\documentclass{article}

    \PassOptionsToPackage{numbers, compress}{natbib}

\usepackage[main,final]{neurips_2026}

\usepackage[utf8]{inputenc} 
\usepackage[T1]{fontenc}    

\usepackage{hyperref}       

\hypersetup{
  colorlinks=true,
  linkcolor=linkpurple,
  citecolor=green!50!black,
  urlcolor=blue!50!black
}
 
\usepackage{url}            
\usepackage{booktabs}       
\usepackage{amsfonts}       
\usepackage{nicefrac}       
\usepackage{microtype}      

\usepackage{amsmath,amssymb}
\usepackage{mathtools}
\usepackage{amsthm}
\usepackage{bm}
\usepackage{enumitem}
\usepackage{ulem}
\usepackage{xcolor}
\usepackage{cleveref}
\usepackage{pifont}

\usepackage{algorithm}
\usepackage{algpseudocode}

\usepackage{multirow}
\usepackage{caption} 
\usepackage{adjustbox}
\usepackage{color, colortbl}
\usepackage{subcaption}
\usepackage{graphicx}
\usepackage{makecell}

\usepackage{wrapfig,lipsum}

\newtheorem{theorem}{Theorem}

\newtheorem{assumption}{Assumption}
\newtheorem{corollary}{Corollary}
\newtheorem{definition}{Definition}
\newtheorem{proposition}{Proposition}

\definecolor{linkpurple}{RGB}{90,40,90}

\definecolor{my_color}{named}{blue}

\definecolor{avgrowgray}{gray}{0.9}
\newcolumntype{a}{>{\columncolor{avgrowgray}}c}

\definecolor{gaegreen}{HTML}{10B981}
\definecolor{gaerowhighlight}{HTML}{E6F9F1}
\makeatletter
\newcommand{\corrfootnote}[1]{%
  \begingroup
  \renewcommand{\@makefntext}[1]{%
    \noindent\makebox[1.8em][r]{*}\,##1%
  }%
  \@footnotetext{#1}%
  \endgroup
}
\makeatother

\title{Geometry-Adaptive Explainer for~Faithful~Dictionary-Based~Interpretability under~Distribution~Shift}

\author{%
Sungjun Lim$^{a}$ \quad
Heedong Kim$^{a}$ \quad
Andrew Lee$^{b, *}$ \quad
Kyungwoo Song$^{a, *}$ \\ [1em]
$^{a}$Yonsei University \qquad
$^{b}$Harvard University
}

\begin{document}

\maketitle

\begingroup
\renewcommand{\thefootnote}{\fnsymbol{footnote}}
\setcounter{footnote}{1}
\footnotetext{Corresponding authors:
\texttt{andrewlee@g.harvard.edu}, \texttt{kyungwoo.song@gmail.com}.}
\endgroup
\setcounter{footnote}{0}

\begin{abstract}
Mechanistic interpretability aims to explain a model's behavior by identifying causally responsible internal structures.
Dictionary-based explainers such as sparse autoencoders and transcoders are a primary tool, but their faithfulness under out-of-distribution (OOD) shift has received little systematic attention.
We show that distribution shift rotates the subspace that the model actively uses, misaligning the explainer's dictionary trained on in-distribution (ID) activations.
We formalize this misalignment as the \textbf{faithfulness gap}, a geometric distance between the ID dictionary and the OOD-active subspace, and show that it bounds excess OOD projection loss.
To reduce this gap, we propose the \textbf{Geometry-Adaptive Explainer (GAE)}, which realigns the dictionary with the OOD-active subspace while preserving the original sparse feature activations.
This requires only unlabeled OOD activations and no gradient updates.
For the adapted dictionary, we also bound the excess projection loss caused by any remaining faithfulness gap.
Empirically, GAE even matches or surpasses all training-based baselines in causal faithfulness across multiple models and OOD settings.\footnote{Code is available at: \url{https://github.com/MLAI-Yonsei/GAE/}}
\end{abstract}

\vspace{-0.5em}
\section{Introduction}\label{sec:introduction}
\vspace{-0.6em}
Mechanistic interpretability aims to explain a model's behavior by identifying internal structures that are causally responsible for its outputs~\cite{geiger2025causal,olsson2022context,sharkey2025open}. A primary approach is to train an explainer, a post-hoc module such as a sparse autoencoder (SAE)~\cite{bricken2023monosemanticity} or transcoder~\cite{dunefsky2024transcoders}, that decomposes hidden activations into sparse combinations of learned feature directions (a dictionary). These dictionary-based explainers have recently been scaled to large language models~\cite{templeton2024scaling,gao2024scaling} and used to uncover interpretable feature circuits~\cite{marks2025sparse}. A central requirement for such explanations is faithfulness: they should accurately reflect the computations the model actually uses~\cite{jacovi2020towards,geiger2025causal}.

When a model encounters out-of-distribution (OOD) inputs, the dictionary learned in-distribution (ID) can no longer capture the directions the model actively uses~\cite{gdm2025negativesaes}. Prior work has addressed this vulnerability from two angles, each limited in scope. Attribution-level robustness studies~\cite{adebayo2018sanity,ghorbani2019interpretation,lin2023robustness,balestra2023consistency} focus on input perturbations rather than hidden-state geometry. Dictionary-explainer remedies such as retraining on the model's own generations~\cite{cho2025faithfulsae}, upweighting rare concepts~\cite{li2020tilted,muhamed2025decoding}, and adding residual modules~\cite{koriagin2025teach} remain heuristic, without diagnosing the underlying misalignment. Consequently, the mechanism driving this failure and a principled correction remain unaddressed.

In this work, we identify a geometric mechanism for OOD faithfulness degradation in dictionary-based explainers.
These explainers learn their feature directions from ID activations, so their dictionaries reflect the geometric structure of ID hidden representations~\cite{bereska2024mechanistic}.
This structure is captured by the second moment of hidden activations, which distribution shift typically alters~\cite{lee2018simple}, leaving the ID-trained dictionary misaligned with the OOD-active subspace.
As illustrated in Figure~\ref{fig:intro} (left), we call this misalignment the \textbf{faithfulness gap}, a geometric distance between the ID dictionary and the OOD-active subspace.
We prove that this gap bounds excess OOD projection loss and that it is itself upper-bounded by the magnitude of the second-moment shift when the explainer is aligned with the ID-active subspace.
Reducing this gap tightens the projection-loss bound, motivating methods that directly realign the dictionary.

\begin{figure}[t]
    \centering
    \includegraphics[width=\linewidth]{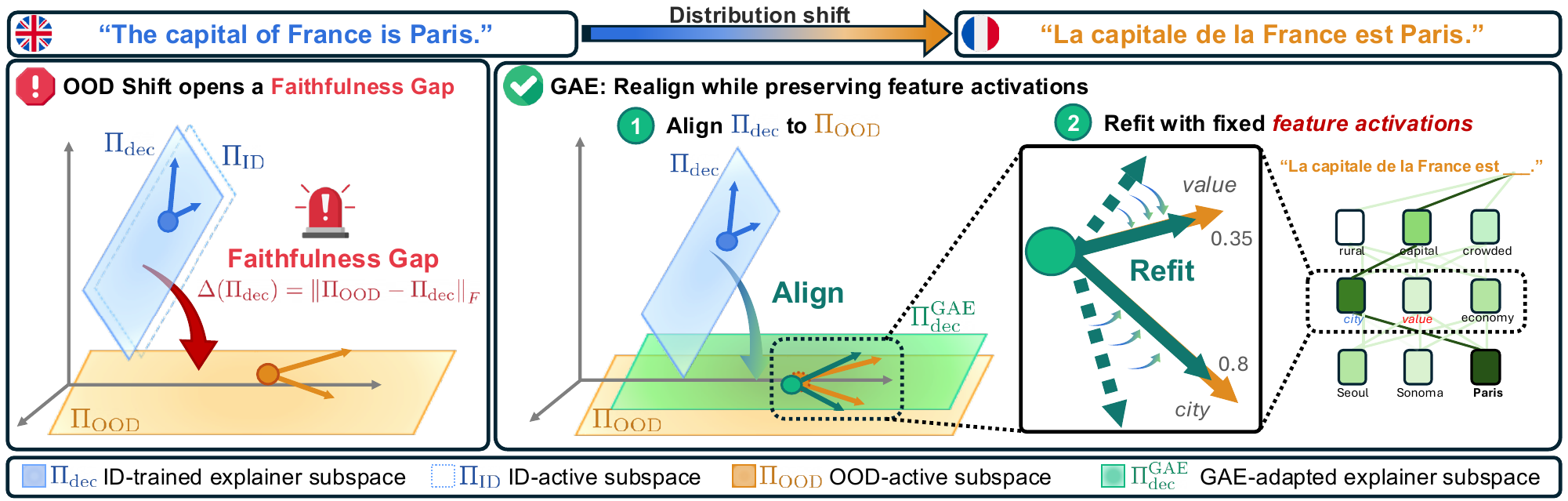}
    \caption{\looseness=-1\textbf{Faithfulness gap and GAE.} \textit{Left}: distribution shift (illustrated as a language change) rotates the OOD-active subspace $\Pi_{\mathrm{OOD}}$ away from the ID-trained explainer subspace $\Pi_{\mathrm{dec}} \approx \Pi_{\mathrm{ID}}$, opening a faithfulness gap $\Delta(\Pi_{\mathrm{dec}})$.
\textit{Right}: GAE uses two steps.
Step~1 rotates $\Pi_{\mathrm{dec}}$ onto the estimated OOD-active subspace via orthogonal Procrustes.
Step~2 refits feature directions under an alignment penalty to match OOD activations while preserving the original sparse feature activations.}
    \label{fig:intro}
    \vspace{-1.5em}
\end{figure}

We instantiate this idea with the \textbf{Geometry-Adaptive Explainer (GAE)}, a closed-form, post-hoc method that targets the faithfulness gap using only unlabeled OOD activations.
GAE first rotates the ID-trained dictionary so that its subspace aligns with the estimated OOD-active subspace, choosing the rotation closest to the original dictionary while preserving the original sparse feature activations (Step~1 in Figure~\ref{fig:intro}).
Consequently, the same input activates the same features with the same strengths before and after adaptation.
A decoder refit then adjusts individual feature directions to match OOD activations while penalizing deviations from this alignment (Step~2 in Figure~\ref{fig:intro}).
The entire pipeline requires no gradient computation.
We show that GAE's projection loss remains close to that of the best subspace of the same dimension when subspace estimation is accurate and decoder refitting introduces little additional misalignment~(Corollary~\ref{cor:gae_projection_bound}).
Empirically, GAE matches or surpasses all training-based baselines in causal faithfulness across multiple language models and OOD settings, including full OOD retraining.

Our main contributions are summarized as follows:
\begin{itemize}[nosep,leftmargin=*]
    \item We identify the \textbf{faithfulness gap} as a geometric mechanism for OOD faithfulness degradation and show how it controls excess OOD projection loss.
    \item We propose \textbf{GAE}, a closed-form dictionary realignment method that targets the faithfulness gap while preserving the original sparse feature activations.
    \item Without any gradient computation, GAE matches or surpasses all training-based baselines, including full OOD retraining, in causal faithfulness across multiple models and diverse OOD settings.
\end{itemize}

\vspace{-0.5em}
\section{Related Work}\label{sec:related_works}
\vspace{-0.5em}
\subsection{Faithfulness in Mechanistic Interpretability}
\vspace{-0.6em}
Dictionary-based explainers such as SAEs~\cite{bricken2023monosemanticity, cunningham2023sparse} and transcoders~\cite{dunefsky2024transcoders,paulo2025transcoders} decompose hidden activations into sparse feature directions.
Faithfulness is evaluated through feature ablation~\cite{jacovi2020towards, edin2025normalized, deyoung2020eraser} and supports reliable circuit discovery and model editing~\cite{marks2025sparse, chan2022causal, meng2022locating}.
Recent empirical reports show that SAE-based features underperform dense baselines on OOD downstream tasks~\cite{gdm2025negativesaes}.

\vspace{-0.6em}
\subsection{Interpretability under Distribution Shift}
\vspace{-0.6em}
Prior work studies the fragility of saliency maps under input perturbations~\cite{adebayo2018sanity, ghorbani2019interpretation} and explanation consistency under shift~\cite{balestra2023consistency, lin2023robustness}.
These analyses focus on attribution methods rather than hidden-space geometry.
OOD detection reveals changes in hidden activations~\cite{lee2018simple}, which representation comparison methods can quantify~\cite{kornblith2019similarity,raghu2017svcca,williams2021generalized}.
The link to explainer faithfulness remains unexplored.

For dictionary-based explainers, proposed remedies include retraining on OOD data, training on the model's own generations~\cite{cho2025faithfulsae}, upweighting tail samples~\cite{li2020tilted, muhamed2025decoding}, or adding residual capacity~\cite{koriagin2025teach}.

Subspace Alignment (SA)~\citep{fernando2013unsupervised} and CORrelation ALignment (CORAL)~\citep{sun2016return} align source and target subspaces or covariance statistics for domain adaptation.
GAE adapts a dictionary decoder for attribution and intervention while keeping the model and sparse encoder frozen.

\vspace{-0.3em}
\section{Hidden-Space Geometry Shift and Faithfulness Degradation}
\label{sec:hidden_geometry_faithfulness}
\vspace{-0.5em}
When a neural network encounters OOD inputs, do its mechanistic explanations remain faithful?
We show that they generally do not.
Distribution shift alters the geometry of hidden activations, creating a misalignment between the directions the model actively uses and those the explainer was trained to reconstruct.
We call this misalignment the faithfulness gap, and prove that \textbf{(i)}~it grows at most proportionally with the second-moment shift for the ID-active subspace (Proposition~\ref{prop:second-moment_shift_enlarges_the_faithfulness_gap}), and \textbf{(ii)}~it controls the reducible part of OOD projection loss (Proposition~\ref{prop:faithfulness_gap_controls_excess_hidden-space_faithfulness}).

\vspace{-0.4em}
\subsection{Setup and Faithfulness Gap}
\label{subsec:notation_preliminaries}
\vspace{-0.5em}
\paragraph{Target model and explainer.}
For an input $X$, the fixed pretrained target model produces the reconstruction target $h(X)\in\mathbb{R}^d$ at a designated layer.
An encoder produces sparse feature activations $z(X)\in\mathbb{R}^k$, and a decoder (dictionary) $W_{\mathrm{dec}}\in\mathbb{R}^{d\times k}$ reconstructs $h(X)$.
SAEs encode $h(X)$ itself, while transcoders encode the MLP input to reconstruct its output $h(X)$; both share the form
\begin{equation}
\label{eq:explainer}
\hat{h}(X) = W_{\mathrm{dec}}\,z(X) + b_{\mathrm{dec}}.
\end{equation}
The encoder applies a sparsifying nonlinearity (e.g., ReLU or TopK)~\citep{cunningham2023sparse}.

\vspace{-0.4em}
\paragraph{In-distribution, out-of-distribution, and second-moment shift.}
Let $P_{\mathrm{ID}}$ and $P_{\mathrm{OOD}}$ denote the ID training and OOD deployment distributions, respectively.

We characterize the shift via second-moment matrix $M_e:=\mathbb{E}_{X\sim P_e}[h(X)h(X)^\top]$, $e\in\{\mathrm{ID},\mathrm{OOD}\}$.
The leading eigenvectors of $M_e$ identify the directions that capture the most activation energy.
Changes in these directions can therefore create a mismatch with the ID-trained dictionary.
A \emph{second-moment shift} occurs when $M_{\mathrm{OOD}} \neq M_{\mathrm{ID}}$, with magnitude $\|M_{\mathrm{OOD}}-M_{\mathrm{ID}}\|_F$.

\vspace{-0.4em}
\paragraph{Active subspace and faithfulness gap.}
Hidden activations concentrate energy along a few directions~\citep{martin2021predicting,ansuini2019intrinsic,aghajanyan2021intrinsic}.
We write $\Pi_e = U_e\,U_e^\top$, where $U_e \in \mathbb{R}^{d \times r}$ contains the top-$r$ eigenvectors of $M_e$, and call this the \emph{active subspace} in environment $e$.
Similarly, $\Pi_{\mathrm{dec}} = U_{\mathrm{dec}}\,U_{\mathrm{dec}}^\top$, where $U_{\mathrm{dec}} \in \mathbb{R}^{d \times r}$ contains the top-$r$ left singular vectors of $W_{\mathrm{dec}}$, is the \emph{explainer subspace}.

Under second-moment shift, $\Pi_{\mathrm{OOD}}$ may diverge from $\Pi_{\mathrm{ID}}$, opening a gap between the explainer subspace and OOD-active subspace.
We call this misalignment the \emph{faithfulness gap}, which controls excess OOD projection loss, as we formalize in Proposition~\ref{prop:faithfulness_gap_controls_excess_hidden-space_faithfulness}:
\begin{definition}[Faithfulness Gap]
\label{def:faithfulness_gap}
The faithfulness gap of an explainer subspace \(\Pi_{\mathrm{dec}}\) under OOD is
\[
\Delta(\Pi_{\mathrm{dec}}):=\|\Pi_{\mathrm{OOD}}-\Pi_{\mathrm{dec}}\|_F.
\]
\end{definition}
For $\Pi_{\mathrm{dec}}=\Pi_{\mathrm{ID}}$, second-moment shift upper-bounds this gap (Proposition~\ref{prop:second-moment_shift_enlarges_the_faithfulness_gap}).


\subsection{Second-Moment Shift Bounds the Faithfulness Gap}
\label{subsec:second-moment_shift_enlarges_the_faithfulness_gap_and_degrades_faithfulness}
\vspace{-0.3em}
The previous subsection defined the faithfulness gap as a geometric quantity.
We now show that second-moment shift upper-bounds this gap for the ID-active subspace.
In practice, explainers are trained on ID activations and deployed without modification~\citep{lieberum2024gemma,mcdougall2025gemmascope2}.
When an ID-trained explainer satisfies $\Pi_{\mathrm{dec}} \approx \Pi_{\mathrm{ID}}$ (empirically assessed in Appendix~\ref{subsec:explainer_subspace_alignment}, Table~\ref{tab:frobenius_gap_real}), its OOD faithfulness depends on how far $\Pi_{\mathrm{ID}}$ lies from $\Pi_{\mathrm{OOD}}$.
The following result bounds this distance.

\begin{proposition}[Second-Moment Shift Bounds the Faithfulness Gap of the ID Explainer]
\label{prop:second-moment_shift_enlarges_the_faithfulness_gap}
Suppose that $M_{\mathrm{ID}}$ has eigengap $\gamma_{\mathrm{ID}}=\lambda_r(M_{\mathrm{ID}})-\lambda_{r+1}(M_{\mathrm{ID}})>0$ at rank $r$. Then
\vspace{-0.6em}
\[
\Delta(\Pi_{\mathrm{ID}})
=
\|\Pi_{\mathrm{OOD}}-\Pi_{\mathrm{ID}}\|_F
\le
\frac{2}{\gamma_{\mathrm{ID}}}\,
\|M_{\mathrm{OOD}}-M_{\mathrm{ID}}\|_F.
\]
\end{proposition}

\vspace{-0.3em}
The faithfulness gap of the ID-active subspace grows at most proportionally with the second-moment shift $\|M_{\mathrm{OOD}} - M_{\mathrm{ID}}\|_F$, with sensitivity controlled by the inverse eigengap $1/\gamma_{\mathrm{ID}}$ (proof in Appendix~\ref{subsec:proof_prop2}; empirical verification in Appendix~\ref{subsec:empirical_verification_propositions}).
Larger shifts permit a larger upper bound on the faithfulness gap.
For an ID dictionary whose projector differs from $\Pi_{\mathrm{ID}}$, the triangle inequality gives $\Delta(\Pi_{\mathrm{dec}})\le\Delta(\Pi_{\mathrm{ID}})+\|\Pi_{\mathrm{dec}}-\Pi_{\mathrm{ID}}\|_F$.
We next relate this gap to excess projection loss.


\subsection{The Faithfulness Gap Controls OOD Degradation}
\label{subsec:what_the_explainer_can_control_under_ood}
\vspace{-0.3em}
We now formalize the connection between the faithfulness gap and OOD faithfulness loss, showing that $\Delta(\Pi_{\mathrm{dec}})$ is the central quantity an adaptation method should target.

\paragraph{OOD faithfulness objective.}
An ideal explainer subspace $\Pi_{\mathrm{dec}}$ minimizes projection-level reconstruction error on OOD activations. We formalize this objective as
\begin{equation}
\label{eq:faithfulness_surrogate}
\mathcal{L}_{\mathrm{OOD}}(\Pi_{\mathrm{dec}})
:=
\mathbb{E}_{X\sim P_{\mathrm{OOD}}}\|h(X)-\Pi_{\mathrm{dec}} h(X)\|_2^2.
\end{equation}
This measures how much OOD activation is lost when projected onto $\Pi_{\mathrm{dec}}$.
$\mathcal{L}_{\mathrm{OOD}} (\Pi_{\mathrm{dec}})$ is a projection-level surrogate for OOD faithfulness~\citep{edin2025normalized,deyoung2020eraser}.

\vspace{-0.2em}
\paragraph{Decomposition of $\mathcal{L}_{\mathrm{OOD}}(\Pi_{\mathrm{dec}})$.}
To isolate the part of the OOD loss that the explainer can reduce, we decompose $\mathcal{L}_{\mathrm{OOD}}(\Pi_{\mathrm{dec}})$ into two terms. Let $\mathcal{C}_r$ denote the set of rank-$r$ orthogonal projectors in $\mathbb{R}^d$. For any $\Pi_{\mathrm{dec}}\in\mathcal{C}_r$,
\begin{equation}
\label{eq:decomposition}
\mathcal{L}_{\mathrm{OOD}}(\Pi_{\mathrm{dec}}) = \underbrace{\mathcal{L}_{\mathrm{OOD}}(\Pi_{\mathrm{OOD}})}_{\text{irreducible}} + \underbrace{\mathcal{L}_{\mathrm{OOD}}(\Pi_{\mathrm{dec}}) - \mathcal{L}_{\mathrm{OOD}}(\Pi_{\mathrm{OOD}})}_{\text{explainer-dependent}}.
\end{equation}
The explainer-dependent component is nonnegative (proof in Appendix~\ref{subsec:proof_decomposition}), so $\Pi_{\mathrm{OOD}}$ minimizes $\mathcal{L}_{\mathrm{OOD}}$ over $\mathcal{C}_r$. The irreducible component depends on the target model and the OOD distribution, both fixed at deployment; adapting the explainer cannot reduce it. \textit{The explainer-dependent component is the only part the explainer can reduce.}

\vspace{-0.2em}
\paragraph{Faithfulness gap as a tight proxy.}
The faithfulness gap $\Delta(\Pi_{\mathrm{dec}})$ measures the distance between two subspaces, making it a direct optimization target. The next proposition shows that controlling $\Delta(\Pi_{\mathrm{dec}})$ is equivalent to controlling the explainer-dependent component.

\begin{proposition}[Faithfulness Gap Controls the Explainer-Dependent Term]
\label{prop:faithfulness_gap_controls_excess_hidden-space_faithfulness}
Assume that the OOD eigengap at rank $r$, $\gamma_{\mathrm{OOD}}=\lambda_r(M_{\mathrm{OOD}})-\lambda_{r+1}(M_{\mathrm{OOD}}) >0$. Then for any $\Pi_{\mathrm{dec}}\in\mathcal{C}_r$,
\[
\frac{\gamma_{\mathrm{OOD}}}{2}\,\Delta(\Pi_{\mathrm{dec}})^2
\le
\mathcal{L}_{\mathrm{OOD}}(\Pi_{\mathrm{dec}})-\mathcal{L}_{\mathrm{OOD}}(\Pi_{\mathrm{OOD}})
\le
\frac{\lambda_1(M_{\mathrm{OOD}})-\lambda_d(M_{\mathrm{OOD}})}{2}\,\Delta(\Pi_{\mathrm{dec}})^2.
\]
\end{proposition}

The lower bound shows that any nonzero gap incurs a positive cost (misalignment cannot be free); the upper bound limits this projection-loss cost quadratically in the gap.
Together, they establish that controlling $\Delta(\Pi_{\mathrm{dec}})$ is equivalent to controlling the explainer-dependent projection loss (proof in Appendix~\ref{subsec:proof_prop1}; empirical verification in Appendix~\ref{subsec:explainer_dependent_magnitude}).
\textit{Reducing $\Delta(\Pi_{\mathrm{dec}})$ tightens an upper bound on excess projection loss, while improvements in nComp and nAOPC are evaluated empirically.}
Section~\ref{sec:gae} introduces a method that directly targets this quantity.

\section{Geometry-Adaptive Explainer (GAE)}
\label{sec:gae}
\vspace{-0.5em}
Section~\ref{sec:hidden_geometry_faithfulness} bounds excess OOD projection loss using the faithfulness gap $\Delta(\Pi_{\mathrm{dec}})$.
We now propose the Geometry-Adaptive Explainer (GAE), which targets $\Delta(\Pi_{\mathrm{dec}})$ by realigning the explainer's dictionary with the OOD-active subspace while preserving the original sparse feature activations.
We present GAE's objective, a closed-form adaptation procedure, and a projection-loss bound.

\vspace{-0.5em}
\subsection{Problem Formulation}
\label{subsec:gae_formulation}
\vspace{-0.5em}
Section~\ref{subsec:second-moment_shift_enlarges_the_faithfulness_gap_and_degrades_faithfulness} bounds the distance between the ID-active and OOD-active subspaces.
Minimizing $\Delta(\Pi_{\mathrm{dec}})$ amounts to choosing a $W_{\mathrm{dec}}$ whose induced subspace aligns with $\Pi_{\mathrm{OOD}}$.
To restore faithfulness, we adapt the existing dictionary $W_{\mathrm{dec}}^{\mathrm{ID}} \in \mathbb{R}^{d \times k}$ using unlabeled OOD pairs $\{(h_i,z_i)\}_{i=1}^N$, estimating the OOD-active subspace $\widehat{\Pi}_{\mathrm{OOD}}$ from $h_i$.
We seek $W_{\mathrm{dec}}$ that closes this gap while preserving the original sparse feature activations:
\begin{equation}
\label{eq:gae_objective}
\min_{W_{\mathrm{dec}}}\;
\underbrace{\mathcal{L}_{\mathrm{recon}}}_{\text{reconstruction}}
\;+\;
\lambda_{\mathrm{geom}}\,
\underbrace{\|\widehat{\Pi}_{\mathrm{OOD}} - \Pi_{\mathrm{dec}}\|_F^2}_{\text{subspace alignment}}
\;+\;
\lambda_{\mathrm{pres}}\,
\underbrace{\|W_{\mathrm{dec}} - W_{\mathrm{dec}}^{\mathrm{ID}}\|_F^2}_{\text{dictionary preservation}},
\end{equation}
where $\mathcal{L}_{\mathrm{recon}} = \frac{1}{N}\sum_{i=1}^N \|h_i - \hat{h}_i\|_2^2$ is the mean reconstruction error.\footnote{The full training objective for the dictionary-based explainer in Equation~\eqref{eq:explainer} includes a sparsity penalty on $\mathbf{z}_i$.
GAE holds the encoder fixed, so $\mathbf{z}_i$ is constant with respect to $W_{\mathrm{dec}}$ and the sparsity term drops out, leaving only the reconstruction term.
This also means GAE applies regardless of sparsity mechanism ($\ell_1$, Top-K~\citep{gao2024scaling}, JumpReLU~\citep{rajamanoharan2024jumprelu}, etc.).}
GAE preserves the original sparse feature activations by keeping the encoder fixed.
For the same input, feature indices, activation values, and sparsity patterns remain unchanged~\citep{bricken2023monosemanticity}.
The three terms serve complementary roles.
The \textit{reconstruction} term fits individual OOD activations, since subspace alignment alone does not guarantee sample-level reconstruction.
The \textit{subspace alignment} term directly targets the faithfulness gap $\Delta(\Pi_{\mathrm{dec}})$, which Proposition~\ref{prop:faithfulness_gap_controls_excess_hidden-space_faithfulness} showed controls the explainer-dependent component.
Dictionary preservation keeps the dictionary close to the original.
Preserving feature activations does not ensure unchanged semantics or causal effects, which depend on decoder directions.

\vspace{-0.5em}
\subsection{Dictionary Adaptation}
\label{subsec:gae_methods}
\vspace{-0.5em}
\looseness=-1
Equation~\eqref{eq:gae_objective} is non-convex in $W_{\mathrm{dec}}$, as $\Pi_{\mathrm{dec}}$ depends on it through a top-$r$ SVD.
GAE uses two closed-form steps (Algorithm~\ref{alg:gae}).
Step~1 enforces \textit{subspace alignment} by rotating the ID-trained dictionary onto $\widehat{\Pi}_{\mathrm{OOD}}$, choosing the rotation closest to the original.
Step~2 refits the decoder via ridge regression, solving the \textit{reconstruction} and \textit{dictionary preservation} terms while penalizing deviations from this alignment.
\vspace{-0.5em}

\begin{algorithm}[H]
\caption{GAE}
\label{alg:gae}
\begin{algorithmic}[1]
\Require $W_{\mathrm{dec}}^{\mathrm{ID}} \in \mathbb{R}^{d \times k}$; OOD pairs $\{(h_i,z_i)\}_{i=1}^N$; $\lambda_{\mathrm{geom}},\, \lambda_{\mathrm{pres}}$
\Ensure $(W_{\mathrm{dec}}^{\mathrm{GAE}},\, b_{\mathrm{dec}}^{\mathrm{GAE}})$
\State $\widehat{M}_{\mathrm{OOD}} \leftarrow \frac{1}{N} \sum_{i=1}^N h_i h_i^\top$
\State $U_{\mathrm{dec}} \leftarrow \text{top-}r\text{ left singular vectors of } W_{\mathrm{dec}}^{\mathrm{ID}}$;\; $U_{\mathrm{OOD}}^{(:r)} \leftarrow \text{top-}r\text{ eigenvectors of } \widehat{M}_{\mathrm{OOD}}$
\State $G \leftarrow U_{\mathrm{dec}}^\top W_{\mathrm{dec}}^{\mathrm{ID}} (W_{\mathrm{dec}}^{\mathrm{ID}})^\top U_{\mathrm{OOD}}^{(:r)}$;\; $T^\star \leftarrow \widetilde{V}\widetilde{U}^\top$ from $\mathrm{SVD}(G)$
\State $\widetilde{W}_{\mathrm{dec}} \leftarrow U_{\mathrm{OOD}}^{(:r)} T^\star U_{\mathrm{dec}}^\top W_{\mathrm{dec}}^{\mathrm{ID}}$ \hfill $\triangleright$ Step 1: Subspace rotation
\State $W_{\mathrm{dec}}^{\mathrm{GAE}} \leftarrow$ Equation~\eqref{eq:gae_step2_dec};\; $b_{\mathrm{dec}}^{\mathrm{GAE}} \leftarrow$ Equation~\eqref{eq:gae_step2_bias} \hfill $\triangleright$ Step 2: Decoder refit
\end{algorithmic}
\end{algorithm}

\vspace{-1.8em}
\paragraph{Step 1: Subspace rotation.}
Let $U_{\mathrm{dec}} \in \mathbb{R}^{d \times r}$ be the top-$r$ left singular vectors of $W_{\mathrm{dec}}^{\mathrm{ID}}$ (the explainer subspace defined in Section~\ref{subsec:notation_preliminaries}), and let $U_{\mathrm{OOD}}^{(:r)}$ be the top-$r$ eigenvectors of the empirical second-moment matrix $\widehat{M}_{\mathrm{OOD}} = \frac{1}{N}\sum_{i=1}^N h_i h_i^\top$. We constrain the rotated dictionary to the form
\begin{equation}
\label{eq:gae_param}
\widetilde{W}_{\mathrm{dec}}(T) \;=\; U_{\mathrm{OOD}}^{(:r)}\; T\;  U_{\mathrm{dec}}^\top\; W_{\mathrm{dec}}^{\mathrm{ID}}, \qquad T \in \mathcal{O}_r,
\end{equation}
where $\mathcal{O}_r = \{T \in \mathbb{R}^{r \times r} : T^\top T = I\}$.
Every column of $\widetilde{W}_{\mathrm{dec}}(T)$ lies in $\mathrm{span}(U_{\mathrm{OOD}}^{(:r)})$, so the column space of $\widetilde{W}_{\mathrm{dec}}(T)$ is contained in this $r$-dimensional subspace.
If $\operatorname{rank}(W_{\mathrm{dec}}^{\mathrm{ID}})\ge r$, $\widetilde{W}_{\mathrm{dec}}(T)$ has rank $r$ and its induced explainer subspace equals $\widehat{\Pi}_{\mathrm{OOD}}$ exactly for any $T \in \mathcal{O}_r$.
Among these rotations, we choose the one that keeps the dictionary closest to the original:
\begin{equation}
\label{eq:gae_obj}
T^\star
\;=\;
\arg\min_{T \in \mathcal{O}_r}
\big\|\widetilde{W}_{\mathrm{dec}}(T) - W_{\mathrm{dec}}^{\mathrm{ID}}\big\|_F^2.
\end{equation}
This is an orthogonal Procrustes problem~\citep{schonemann1966generalized}.
Let $G = U_{\mathrm{dec}}^\top\, W_{\mathrm{dec}}^{\mathrm{ID}}\, (W_{\mathrm{dec}}^{\mathrm{ID}})^\top\, U_{\mathrm{OOD}}^{(:r)} \in \mathbb{R}^{r \times r}$, with SVD $G = \widetilde{U}\Sigma\widetilde{V}^\top$.
Then $T^\star = \widetilde{V}\widetilde{U}^\top$ (derivation in Appendix~\ref{subsec:derivation_gae_procrustes}).
Since the Step~1 dictionary has projector $\widehat{\Pi}_{\mathrm{OOD}}$ by construction, the residual faithfulness gap reduces to the eigenspace estimation error.
This yields a quantitative improvement over the unadapted ID explainer.

\begin{theorem}[Improvement over ID Explainer]
\label{thm:improvement_over_id}
Suppose $\gamma_{\mathrm{OOD}} := \lambda_r(M_{\mathrm{OOD}}) - \lambda_{r+1}(M_{\mathrm{OOD}}) > 0$, $\Delta(\Pi_{\mathrm{ID}}) > 0$, and $\widehat{\Pi}_{\mathrm{OOD}} = \Pi_{\mathrm{OOD}}$.
Then
\begin{equation}
\label{eq:gae_improvement_over_id}
\mathcal{L}_{\mathrm{OOD}}(\widehat{\Pi}_{\mathrm{OOD}})
\;\le\;
\mathcal{L}_{\mathrm{OOD}}(\Pi_{\mathrm{ID}})
\;-\;
\frac{\gamma_{\mathrm{OOD}}}{2}\,\Delta(\Pi_{\mathrm{ID}})^2.
\end{equation}
\end{theorem}

\vspace{-0.5em}
This guarantees a Step~1 projection-loss improvement over the ID-active projector under exact subspace estimation.
The proof is given in Appendix~\ref{subsec:proof_theorem_improvement}.

\vspace{-0.6em}
\paragraph{Step 2: Decoder refit.}
Step~1 aligns the subspace but does not optimize sample-level reconstruction.
Step~2 refits $\mathcal{L}_{\mathrm{recon}}$ while penalizing deviations from the Step~1 alignment.
With fixed feature activations $z_i$, reconstruction reduces to a linear least-squares problem in $W_{\mathrm{dec}}$ and $b_{\mathrm{dec}}$.
To keep the decoder geometrically aligned, we penalize decoder mass outside $\widehat{\Pi}_{\mathrm{OOD}}$ with $\lambda_{\mathrm{geom}}\|(I - \widehat{\Pi}_{\mathrm{OOD}})\,W_{\mathrm{dec}}\|_F^2$.
To limit displacement from the aligned dictionary, the dictionary preservation penalty regularizes the decoder toward $\widetilde{W}_{\mathrm{dec}}(T^\star)$ with $\lambda_{\mathrm{pres}}\|W_{\mathrm{dec}} - \widetilde{W}_{\mathrm{dec}}(T^\star)\|_F^2$.
The combined objective is convex and quadratic, yielding the closed-form solution
\begin{equation}
\label{eq:gae_step2_dec}
W_{\mathrm{dec}}^{\mathrm{GAE}} \;=\; \widehat{\Pi}_{\mathrm{OOD}}\, C\, B^{-1} \;+\; (I - \widehat{\Pi}_{\mathrm{OOD}})\, C\, (B + \lambda_{\mathrm{geom}}\, I)^{-1},
\end{equation}
\begin{equation}
\label{eq:gae_step2_bias}
b_{\mathrm{dec}}^{\mathrm{GAE}} \;=\; \tfrac{1}{N}\textstyle\sum_i h_i \;-\; W_{\mathrm{dec}}^{\mathrm{GAE}}\;\tfrac{1}{N}\textstyle\sum_i z_i,
\end{equation}
\vspace{-0.5em}
where
\vspace{-0.5em}
\begin{align}
B &= \tfrac{1}{N}\textstyle\sum_i z_i z_i^\top \;-\; \bigl(\tfrac{1}{N}\textstyle\sum_i z_i\bigr)\bigl(\tfrac{1}{N}\textstyle\sum_i z_i\bigr)^{\!\top} \;+\; \lambda_{\mathrm{pres}}\, I, \label{eq:gae_B} \\
C &= \tfrac{1}{N}\textstyle\sum_i h_i z_i^\top \;-\; \bigl(\tfrac{1}{N}\textstyle\sum_i h_i\bigr)\bigl(\tfrac{1}{N}\textstyle\sum_i z_i\bigr)^{\!\top} \;+\; \lambda_{\mathrm{pres}}\, \widetilde{W}_{\mathrm{dec}}(T^\star). \label{eq:gae_C}
\end{align}
The Step~2 solution applies regularization strength $\lambda_{\mathrm{pres}}$ to decoder mass inside the OOD-active subspace and the larger strength $\lambda_{\mathrm{pres}} + \lambda_{\mathrm{geom}}$ outside it, so \textit{any decoder mass that drifts off the Step~1 alignment is automatically shrunk back}.
The full derivation is in Appendix~\ref{subsec:derivation_gae_step2}.

\vspace{-0.6em}
\paragraph{Applying GAE at inference.}
At inference, the ID-trained encoder remains unchanged: it extracts sparse feature activations $z$ from each OOD example.
Using the refitted decoder, GAE (Algorithm~\ref{alg:gae}) reconstructs $\hat{h} = W_{\mathrm{dec}}^{\mathrm{GAE}}\,z + b_{\mathrm{dec}}^{\mathrm{GAE}}$.
To control the effect of decoder refitting, we interpolate its output with the Step~1 dictionary:
\begin{equation}
\label{eq:decoder_interpolation}
W_{\mathrm{dec}} = (1 - \alpha)\,W_{\mathrm{dec}}^{\mathrm{GAE}} + \alpha\,\widetilde{W}_{\mathrm{dec}}(T^\star),
\end{equation}
where $\alpha \in [0,1]$ is the weight of the Step~1 dictionary, selected on held-out OOD activations (Appendix~\ref{subsec:gae_implementation_details}).
For $\alpha=0$, GAE uses the refitted decoder.
For $\alpha=1$, Step~1-Only directly uses $W_{\mathrm{dec}}=\widetilde{W}_{\mathrm{dec}}(T^\star)$ and skips Step~2.
This applies identically to both SAEs and transcoders.

\vspace{-0.5em}
\subsection{Projection-Loss Bound for GAE}
\label{subsec:gae_projection_bound}
\vspace{-0.5em}
GAE's projection loss remains close to the optimum over rank-$r$ subspaces when subspace estimation is accurate and decoder refitting introduces little additional misalignment.
Step~1 aligns the dictionary with the estimated OOD-active subspace, while Step~2 can alter this alignment.

We quantify the subspace estimation error by $\varepsilon_N:=\|\widehat{\Pi}_{\mathrm{OOD}}-\Pi_{\mathrm{OOD}}\|_F$.
For decoder refitting, define $R_\perp:=(I-\widehat{\Pi}_{\mathrm{OOD}})C$ using $C$ from Equation~\eqref{eq:gae_C}.
Equation~\eqref{eq:gae_step2_dec} shows that $R_\perp$ drives the decoder component outside this subspace.
The Step~2 penalties limit this component, and interpolation scales it by $1-\alpha$.
Under the rank condition below, this control bounds the additional subspace misalignment.
Combining this bound with the subspace estimation error and Proposition~\ref{prop:faithfulness_gap_controls_excess_hidden-space_faithfulness} gives GAE's excess projection-loss bound.

\begin{corollary}[Projection Loss of GAE]
\label{cor:gae_projection_bound}
Let $\beta_{\mathrm{OOD}}:=\lambda_1(M_{\mathrm{OOD}})-\lambda_d(M_{\mathrm{OOD}})$, and write $\sigma_r(\cdot)$ for the $r$-th singular value.
Suppose $\gamma_{\mathrm{OOD}}>0$, $\operatorname{rank}(W_{\mathrm{dec}}^{\mathrm{ID}})\ge r$, and $0\le\alpha<1$.
Let $\lambda_{\mathrm{pres}},\lambda_{\mathrm{geom}}\ge0$ with $\lambda_{\mathrm{pres}}+\lambda_{\mathrm{geom}}>0$, and suppose $B$ is invertible.
For the dictionary $W_{\mathrm{dec}}$ in Equation~\eqref{eq:decoder_interpolation} and its top-$r$ projector $\Pi_{\mathrm{dec}}$, assume $\sigma_r(\widehat{\Pi}_{\mathrm{OOD}}W_{\mathrm{dec}})>0$.
Then
\begin{equation}
\label{eq:gae_projection_bound}
\mathcal{L}_{\mathrm{OOD}}(\Pi_{\mathrm{dec}})-\mathcal{L}_{\mathrm{OOD}}(\Pi_{\mathrm{OOD}})
\;\le\;
\frac{\beta_{\mathrm{OOD}}}{2}
\left(\varepsilon_N+
\frac{2(1-\alpha)\|R_\perp\|_F}
{(\lambda_{\mathrm{pres}}+\lambda_{\mathrm{geom}})\,\sigma_r(\widehat{\Pi}_{\mathrm{OOD}}W_{\mathrm{dec}})}
\right)^2.
\end{equation}
\end{corollary}

Thus, GAE's projection loss can stay close to the optimum even after refitting, provided both subspace estimation error and additional misalignment remain small.
For $\alpha=1$, Step~1-Only skips Step~2 and has projector $\widehat{\Pi}_{\mathrm{OOD}}$, giving excess projection loss at most $\beta_{\mathrm{OOD}}\varepsilon_N^2/2$.
When $\varepsilon_N=0$, this is the Step~1 endpoint of Theorem~\ref{thm:improvement_over_id}.
The proof is given in Appendix~\ref{subsec:proof_gae_projection_bound}.

\vspace{-0.5em}
\section{Experiments}\label{sec:experiments}
\vspace{-0.5em}
We evaluate whether GAE restores explanation faithfulness under distribution shift. Section~\ref{subsec:controlled_validation} tests the geometric mechanism in a controlled setting, and Section~\ref{subsec:lm_experiments} evaluates on language models. 
\vspace{-0.5em}

\vspace{-0.4em}
\subsection{Controlled Experiment}
\label{subsec:controlled_validation}
\vspace{-0.4em}

\begin{figure}[ht]
\centering
\begin{subfigure}[t]{0.48\linewidth}
\centering
\includegraphics[width=0.90\linewidth]{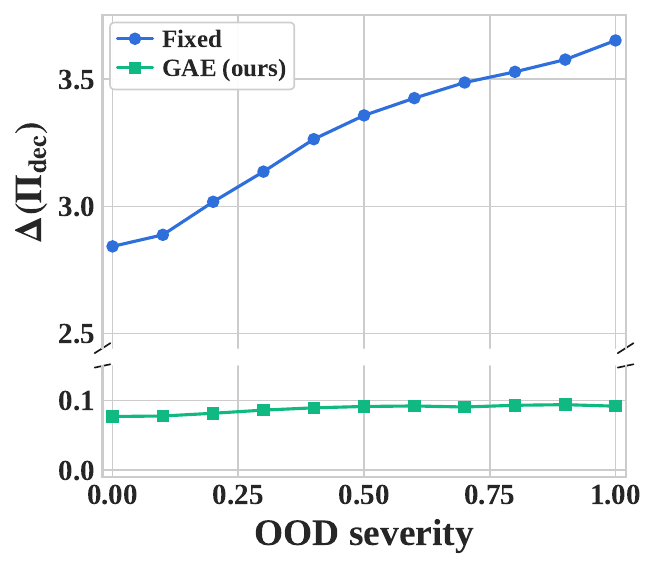}
\caption{Faithfulness gap.}
\label{fig:controlled_validation_a}
\end{subfigure}
\hfill
\begin{subfigure}[t]{0.48\linewidth}
\centering
\includegraphics[width=0.90\linewidth]{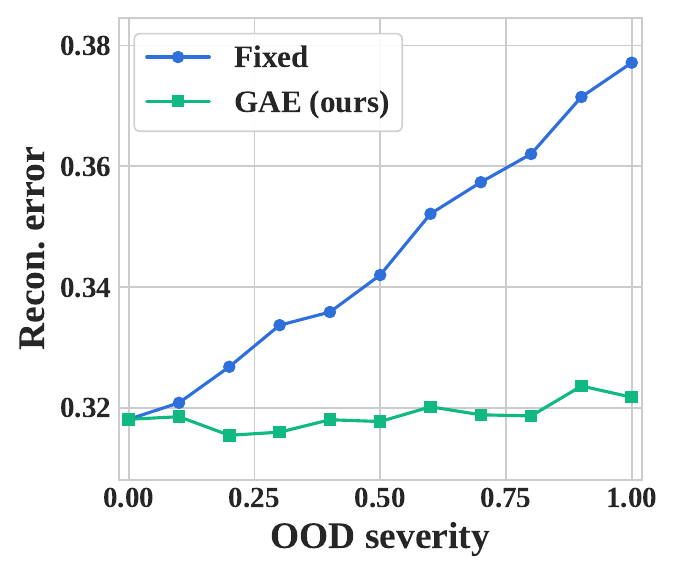}
\caption{Reconstruction error.}
\label{fig:controlled_validation_b}
\end{subfigure}
\caption{\textbf{Controlled experiment on a toy MLP with OOD severity varied from 0 (ID) to 1 (maximum shift).} (a)~The Fixed explainer's faithfulness gap $\Delta(\Pi_{\mathrm{dec}})$ grows monotonically. (b)~Its reconstruction error rises accordingly. GAE maintains near-zero gap and flat error throughout.}
\label{fig:controlled_validation}
\end{figure}

We first test whether the geometric mechanism from Section~\ref{sec:hidden_geometry_faithfulness} holds in a controlled setting. We train a 2-layer ReLU MLP with hidden dim $d\!=\!256$ and output dim $p\!=\!8$, and a linear-decoder SAE on its ID hidden activations, then continuously increase OOD severity by rotating and rescaling the input covariance (details in Appendix~\ref{subsec:toy_setting}). Figure~\ref{fig:controlled_validation} confirms that as severity increases, the faithfulness gap and reconstruction error of the Fixed explainer (the ID-trained explainer, used without adaptation) grow, while GAE closes the gap to near zero and keeps reconstruction error nearly flat.

\subsection{Experiments on Language Models}
\label{subsec:lm_experiments}
\subsubsection{Setup}
\label{subsec:experimental_setup}
\paragraph{Target models and explainers.}
We use two main frozen pretrained language models: GPT-2 Small~\citep{radford2019language} and Pythia-1.4B~\citep{biderman2023pythia}.
For each model, we train transcoders~\citep{dunefsky2024transcoders} and Top-K SAEs~\citep{gao2024scaling}, both with dictionary size $k{=}32d$, as dictionary-based explainers (Equation~\eqref{eq:explainer}).

\vspace{-0.5em}
\paragraph{OOD settings.}
We consider three categories of distribution shift: \textbf{temporal} (FineWeb~\citep{penedo2024fineweb}, web text collected after each model's pretraining cutoff), \textbf{domain} (Edgar~\citep{loukas2021edgar}, financial filings whose specialized vocabulary and structure differ from general web text), and \textbf{adversarial} (HaluEval~\citep{li2023halueval}, hallucination-inducing prompts that elicit atypical hidden representations). All three induce measurable second-moment shift in hidden activations, as verified in Appendix~\ref{subsec:empirical_verification_propositions}.

\vspace{-0.7em}
\paragraph{Baselines.}
We compare training-free and training-based approaches. \textbf{Fixed} applies the ID-trained explainer without adaptation. \textbf{TERM}~\citep{li2020tilted,muhamed2025decoding} trains the ID explainer with tilted ERM to upweight tail samples. Among training-based methods, \textbf{Finetune}~\citep{kissane2024saetransfer} warm-starts from the ID explainer on OOD activations, \textbf{Retrain} trains from scratch on OOD data, \textbf{SAEBoost}~\citep{koriagin2025teach} adds a residual booster on OOD reconstruction residuals, and \textbf{FaithfulSAE}~\citep{cho2025faithfulsae} retrains on the model's own generations. \textbf{GAE (ours)} is training-free: it uses only unlabeled OOD activations with no gradient computation. Detailed descriptions are in Appendix~\ref{subsec:baseline_details}.

\vspace{-0.5em}
\paragraph{Evaluation metrics.}
We evaluate causal faithfulness and reconstruction quality using three metrics.
\textbf{Normalized AOPC (nAOPC)}~\citep{edin2025normalized} averages the normalized logit drop across multiple feature budgets when top-$m$ features are removed ($\uparrow$ is better).
\textbf{Normalized comprehensiveness (nComp)}~\citep{deyoung2020eraser} measures the normalized logit drop at a single budget $m^*{=}32$ ($\uparrow$ is better).
Both measure the logit-level effect of ablating top features.
\textbf{Delta cross-entropy ($\Delta$CE)}~\citep{gao2024scaling} measures reconstruction quality: the cross-entropy change when activations are replaced with the explainer's reconstruction ($\approx 0$ is better).
GAE optimizes a geometric objective (the faithfulness gap); improvements on the two causal metrics test whether geometric realignment yields faithfulness gains.
Formal definitions are in Appendix~\ref{subsec:evaluation_details}.

\noindent
\begin{minipage}[t]{0.51\linewidth}
\vspace{0pt}
\looseness=-1
As shown in Table~\ref{tab:compute_cost}, all training-based baselines require gradient-based optimization over millions of tokens.
Even the lightest, Finetune, processes 5M tokens over several minutes; Retrain, SAEBoost, and FaithfulSAE each consume 100M tokens and take tens of minutes to hours on a single GPU.
\textit{GAE requires no gradient computation at all}: one closed-form adaptation run completes in 0.5\,s for GPT-2 and 2.9\,s for Pythia-1.4B, using only ${\sim}2{,}048$ unlabeled OOD activations.
\end{minipage}
\hfill
\begin{minipage}[t]{0.47\linewidth}
\vspace{0pt}
\centering
\captionsetup{position=top,skip=6pt}
\captionof{table}{\textbf{Computational cost per method.}
We report token counts and wall-clock times.
GAE takes under 3\,s per model without training.}
\label{tab:compute_cost}
\adjustbox{width=1.0\linewidth}{
\begin{tabular}{lccc}
\toprule
\textbf{Method} & \makecell[c]{\textbf{Tokens}\\ \textbf{(Vectors)}} & \textbf{GPT-2} & \textbf{P-1.4B} \\
\midrule
Finetune & 5M & ${\sim}$2\,min & ${\sim}$12\,min \\
Retrain & 100M & ${\sim}$39\,min & ${\sim}$4\,hrs \\
SAEBoost & 100M & ${\sim}$39\,min & ${\sim}$4\,hrs \\
FaithfulSAE & 100M & ${\sim}$39\,min & ${\sim}$4\,hrs \\
\midrule
\rowcolor{gaerowhighlight} \textbf{GAE} & \textbf{(2K)} & \textbf{0.5\,s} & \textbf{2.9\,s} \\
\bottomrule
\end{tabular}
}
\end{minipage}

\vspace{-0.5em}
\subsubsection{Faithfulness Results}
\label{subsec:main_results}
\vspace{-0.5em}

\begin{table*}[h]
\centering
\caption{\textbf{Faithfulness under distribution shift (Transcoder, two models $\times$ three OOD settings).} GAE is training-free yet leads in all nine columns on GPT-2 and achieves the best nComp and $|\Delta\mathrm{CE}|$ in 5 of 6 Pythia-1.4B columns. \textbf{Bold}: best per column. \underline{Underline}: second best.}
\label{tab:main_results}
\vspace{-0.5em}
\adjustbox{max width=\textwidth}{
\begin{tabular}{l ccc ccc ccc}
\toprule
& \multicolumn{3}{c}{\textbf{FineWeb (Temporal)}} & \multicolumn{3}{c}{\textbf{Edgar (Domain)}} & \multicolumn{3}{c}{\textbf{HaluEval (Adversarial)}} \\
\cmidrule(lr){2-4} \cmidrule(lr){5-7} \cmidrule(lr){8-10}
\textbf{Method} & \textbf{nAOPC}$\uparrow$ & \textbf{nComp}$\uparrow$ & $|\boldsymbol{\Delta}\textbf{CE}|\!\downarrow$ & \textbf{nAOPC}$\uparrow$ & \textbf{nComp}$\uparrow$ & $|\boldsymbol{\Delta}\textbf{CE}|\!\downarrow$ & \textbf{nAOPC}$\uparrow$ & \textbf{nComp}$\uparrow$ & $|\boldsymbol{\Delta}\textbf{CE}|\!\downarrow$ \\
\midrule
\multicolumn{10}{l}{\textbf{\textit{GPT-2 Small}}} \\ \addlinespace
Fixed           & 0.857 & 1.017 & 0.0281 & 0.975 & 1.025 & 0.0201 & 0.735 & 0.737 & 0.0473 \\
TERM            & 0.853 & 0.993 & 0.0283 & 0.964 & 0.999 & 0.0198 & 0.730 & 0.732 & 0.0523 \\
Finetune        & 0.856 & 0.984 & \underline{0.0172} & 0.971 & 1.117 & 0.0047 & 0.579 & 0.579 & \underline{0.0105} \\
Retrain         & 0.895 & 1.118 & 0.0218 & 0.936 & 1.034 & \underline{0.0015} & 0.521 & 0.528 & 0.2763 \\
SAEBoost        & \underline{0.958} & \underline{1.476} & 0.0177 & \underline{0.979} & \underline{1.542} & 0.0072 & \underline{0.840} & \underline{0.921} & 0.0212 \\
FaithfulSAE     & 0.936 & 1.186 & 0.0213 & 0.976 & 1.134 & 0.0197 & 0.738 & 0.740 & 0.0586 \\
\midrule
\rowcolor{gaerowhighlight} \textbf{GAE}      & \textbf{0.960} & \textbf{1.494} & \textbf{0.0167} & \textbf{0.981} & \textbf{1.618} & \textbf{0.0009} & \textbf{0.871} & \textbf{0.963} & \textbf{0.0014} \\
\midrule\midrule
\multicolumn{10}{l}{\textbf{\textit{Pythia-1.4B}}} \\ \addlinespace
Fixed           & 0.839 & 1.091 & 0.0278 & 0.725 & 0.828 & 0.0300 & 0.899 & 1.021 & 0.0354 \\
TERM            & 0.845 & 1.043 & 0.0271 & 0.895 & 0.981 & 0.0298 & 0.896 & 1.104 & 0.0349 \\
Finetune        & 0.859 & 1.249 & \textbf{0.0264} & 0.684 & 0.860 & \underline{0.0282} & 0.925 & 1.502 & \underline{0.0280} \\
Retrain         & 0.894 & \underline{1.315} & 0.0405 & 0.746 & 0.821 & 0.0329 & 0.894 & 1.393 & 0.0305 \\
SAEBoost        & \underline{0.908} & 1.296 & 0.0284 & \underline{0.903} & \underline{1.297} & 0.0296 & \underline{0.965} & \underline{1.530} & 0.0283 \\
FaithfulSAE     & 0.858 & 1.056 & 0.0272 & 0.724 & 0.822 & 0.0323 & 0.899 & 1.171 & 0.0307 \\
\midrule
\rowcolor{gaerowhighlight} \textbf{GAE}      & \textbf{0.915} & \textbf{1.354} & \underline{0.0269} & \textbf{0.988} & \textbf{1.652} & \textbf{0.0230} & \textbf{0.968} & \textbf{1.693} & \textbf{0.0276} \\
\bottomrule
\end{tabular}}
\vspace{-0.5em}
\end{table*}

Table~\ref{tab:main_results} reports faithfulness across three OOD settings for GPT-2 Small and Pythia-1.4B (Transcoder).
Despite using no gradient updates, \textit{GAE leads on all three metrics for GPT-2 across every OOD setting}, surpassing training-based baselines that consume up to 100M tokens (cf.\ Table~\ref{tab:compute_cost}).
On adversarial shift, GAE improves nComp over the strongest training-based baseline on this metric, SAEBoost, by 4.6\% (0.963 vs 0.921).
Relative to the same baseline, GAE reduces $|\Delta\mathrm{CE}|$ by 93\% (0.0014 vs 0.0212).
On Pythia-1.4B, GAE achieves the best nAOPC and nComp on all three settings; $|\Delta\mathrm{CE}|$ is best on Edgar and HaluEval but slightly elevated on FineWeb (0.027 vs Finetune's 0.026), where the more diffuse eigenvalue decay weakens the rank-$r$ approximation.

\vspace{-0.5em}
\begin{table*}[h]
\centering
\caption{\textbf{Faithfulness under distribution shift (SAE, two models $\times$ three OOD settings).} GAE leads in all nine columns on GPT-2 and leads on nComp and $|\Delta\mathrm{CE}|$ in 5 of 6 Pythia-1.4B columns. \textbf{Bold}: best. \underline{Underline}: second best.}
\label{tab:main_results_sae}
\vspace{-0.5em}
\adjustbox{max width=\textwidth}{
\begin{tabular}{l ccc ccc ccc}
\toprule
& \multicolumn{3}{c}{\textbf{FineWeb (Temporal)}} & \multicolumn{3}{c}{\textbf{Edgar (Domain)}} & \multicolumn{3}{c}{\textbf{HaluEval (Adversarial)}} \\
\cmidrule(lr){2-4} \cmidrule(lr){5-7} \cmidrule(lr){8-10}
\textbf{Method} & \textbf{nAOPC}$\uparrow$ & \textbf{nComp}$\uparrow$ & $|\boldsymbol{\Delta}\textbf{CE}|\!\downarrow$ & \textbf{nAOPC}$\uparrow$ & \textbf{nComp}$\uparrow$ & $|\boldsymbol{\Delta}\textbf{CE}|\!\downarrow$ & \textbf{nAOPC}$\uparrow$ & \textbf{nComp}$\uparrow$ & $|\boldsymbol{\Delta}\textbf{CE}|\!\downarrow$ \\
\midrule
\multicolumn{10}{l}{\textbf{\textit{GPT-2 Small}}} \\ \addlinespace
Fixed           & 0.735 & 0.796 & 0.0185 & 0.650 & 0.667 & 0.0089 & 0.930 & 1.134 & 0.0406 \\
TERM            & 0.741 & 0.790 & 0.0183 & 0.655 & 0.682 & 0.0145 & 0.898 & 0.987 & 0.0401 \\
Finetune        & 0.725 & 0.802 & \underline{0.0015} & 0.658 & 0.687 & 0.0068 & \underline{0.932} & 1.155 & \underline{0.0218} \\
Retrain         & \underline{0.766} & \underline{0.856} & 0.0375 & \underline{0.715} & \underline{0.797} & \underline{0.0065} & 0.930 & \underline{1.276} & 0.0300 \\
SAEBoost        & 0.704 & 0.786 & 0.0120 & 0.604 & 0.650 & 0.0098 & 0.909 & 1.243 & 0.0448 \\
FaithfulSAE     & 0.725 & 0.760 & 0.0262 & 0.657 & 0.683 & 0.0074 & 0.908 & 1.082 & 0.0278 \\
\midrule
\rowcolor{gaerowhighlight} \textbf{GAE} & \textbf{0.768} & \textbf{0.871} & \textbf{0.0011} & \textbf{0.723} & \textbf{0.809} & \textbf{0.0037} & \textbf{0.953} & \textbf{1.303} & \textbf{0.0017} \\
\midrule\midrule
\multicolumn{10}{l}{\textbf{\textit{Pythia-1.4B}}} \\ \addlinespace
Fixed           & 0.962 & 1.536 & 0.0216 & 0.953 & 1.690 & 0.0170 & 0.985 & 1.425 & 0.0252 \\
TERM            & 0.965 & 1.486 & 0.0237 & 0.959 & 1.787 & 0.0167 & 0.984 & 1.762 & 0.0292 \\
Finetune        & 0.963 & 1.618 & 0.0216 & 0.959 & 1.832 & \underline{0.0131} & \underline{0.988} & \underline{1.880} & 0.0174 \\
Retrain         & \underline{0.983} & \textbf{1.681} & 0.0563 & \textbf{0.970} & \underline{1.848} & 0.0261 & \textbf{1.000} & 1.855 & 0.0209 \\
SAEBoost        & 0.971 & 1.451 & \underline{0.0211} & 0.955 & 1.598 & 0.0152 & \textbf{1.000} & 1.830 & \underline{0.0166} \\
FaithfulSAE     & 0.982 & 1.642 & 0.0307 & 0.953 & 1.769 & 0.0205 & \textbf{1.000} & 1.678 & 0.0514 \\
\midrule
\rowcolor{gaerowhighlight} \textbf{GAE} & \textbf{0.985} & \underline{1.677} & \textbf{0.0207} & \underline{0.968} & \textbf{1.946} & \textbf{0.0098} & \textbf{1.000} & \textbf{1.885} & \textbf{0.0163} \\
\bottomrule
\end{tabular}}
\end{table*}

Table~\ref{tab:main_results_sae} reports the same evaluation for Top-K SAEs.
On GPT-2, GAE again surpasses every training-based baseline on all nine columns, including gains on Edgar (nComp 0.809 vs Retrain's 0.797) and HaluEval ($|\Delta\mathrm{CE}|$ 0.0017 vs Finetune's 0.0218).
On Pythia-1.4B, Retrain is a stronger competitor than for transcoders, taking best nComp on FineWeb (1.681 vs 1.677) and best nAOPC on Edgar (0.970 vs 0.968).
GAE still leads on nComp and $|\Delta\mathrm{CE}|$ in 5 of 6 columns and achieves the best $|\Delta\mathrm{CE}|$ across all three settings.
The pattern is consistent: \textit{GAE matches or surpasses methods that require orders of magnitude more computation.}

\paragraph{Comparison with geometric adaptation.}
We compare geometric adaptation methods to show that GAE's decoder adaptation for causal faithfulness outperforms generic geometric alignment under OOD shift.
SA~\citep{fernando2013unsupervised} aligns source and target PCA subspaces, whereas CORAL~\citep{sun2016return} matches covariance statistics.
Both decoder adaptations keep the sparse encoder frozen.
They and dense PCA use same unlabeled OOD activations and evaluation protocol.
GAE leads on nAOPC, nComp, and $|\Delta\mathrm{CE}|$ (Table~\ref{tab:uda_results}).
Appendix~\ref{subsec:uda_comparison} also reports the result on SAE.

\begin{table}[ht]
\vspace{-1.5em}
\centering
\begin{minipage}[t]{0.60\textwidth}
\centering
\captionsetup{font=footnotesize}
\caption{\textbf{Geometric adaptation on HaluEval (Transcoder).}
GAE improves causal faithfulness and reconstruction over PCA, SA, and CORAL without gradient updates.
\textbf{Bold}: best. \underline{Underline}: second best.}
\label{tab:uda_results}
\footnotesize
\setlength{\tabcolsep}{2pt}
\adjustbox{max width=\linewidth}{
\begin{tabular}{lccc ccc}
\toprule
\multirow{2}{*}[-2.5pt]{\textbf{Method}} & \multicolumn{3}{c}{\textbf{GPT-2 Small}} & \multicolumn{3}{c}{\textbf{Pythia-1.4B}} \\
\cmidrule(lr){2-4} \cmidrule(lr){5-7}
& \textbf{nAOPC}$\uparrow$ & \textbf{nComp}$\uparrow$ & $|\boldsymbol{\Delta}\textbf{CE}|\!\downarrow$ & \textbf{nAOPC}$\uparrow$ & \textbf{nComp}$\uparrow$ & $|\boldsymbol{\Delta}\textbf{CE}|\!\downarrow$ \\
\midrule
Fixed & 0.735 & 0.737 & 0.0473 & 0.899 & 1.021 & \underline{0.0354} \\
PCA & 0.784 & \underline{0.910} & 0.0808 & \underline{0.942} & \underline{1.364} & 0.0581 \\
SA & \underline{0.801} & 0.895 & \underline{0.0415} & 0.536 & 0.544 & 0.0816 \\
CORAL & 0.405 & 0.409 & 0.0580 & 0.931 & 1.305 & 0.1048 \\
\rowcolor{gaerowhighlight} \textbf{GAE} & \textbf{0.871} & \textbf{0.963} & \textbf{0.0014} & \textbf{0.968} & \textbf{1.693} & \textbf{0.0276} \\
\bottomrule
\end{tabular}
}
\end{minipage}
\hfill
\begin{minipage}[t]{0.38\textwidth}
\centering
\captionsetup{font=footnotesize}
\caption{\textbf{Llama-3.1-8B SAE on HaluEval.}
GAE leads on all metrics without gradient updates.
\textbf{Bold}: best. \underline{Underline}: second best.}
\label{tab:llama_results}
\footnotesize
\setlength{\tabcolsep}{3pt}
\renewcommand{\arraystretch}{1.5}
\adjustbox{max width=\linewidth}{
\begin{tabular}{lccc}
\toprule
\textbf{Method} & \textbf{nAOPC}$\uparrow$ & \textbf{nComp}$\uparrow$ & $|\boldsymbol{\Delta}\textbf{CE}|\!\downarrow$ \\
\midrule
Fixed & 0.5719 & 0.7892 & 0.0652 \\
Finetune & \underline{0.5906} & \underline{0.7913} & \underline{0.0609} \\
SAEBoost & 0.5828 & 0.7902 & 0.0631 \\
\rowcolor{gaerowhighlight} \textbf{GAE} & \textbf{0.5972} & \textbf{0.7936} & \textbf{0.0602} \\
\bottomrule
\end{tabular}
}
\end{minipage}
\vspace{-1.5em}
\end{table}

\paragraph{Larger-model evaluation.}
We evaluate Llama-3.1-8B with a public LlamaScope 32$\times$ Top-K SAE~\citep{he2024llamascope} containing 128K features on HaluEval. 
Using $r=64$ and $N=2{,}048$ unlabeled OOD activation vectors, GAE adapts the decoder while keeping the model and dictionary encoder frozen.
Table~\ref{tab:llama_results} shows higher nAOPC and nComp and lower $|\Delta\mathrm{CE}|$ than Fixed, Finetune, and SAEBoost.
Appendix~\ref{sec:larger_model} provides the setup and computational cost details.

\vspace{-0.2em}
\subsubsection{Case Study: Circuit Attribution under Distribution Shift}
\label{subsec:case_study}

\begin{figure}[h]
\vspace{-1.0em}
\centering
\includegraphics[width=\linewidth]{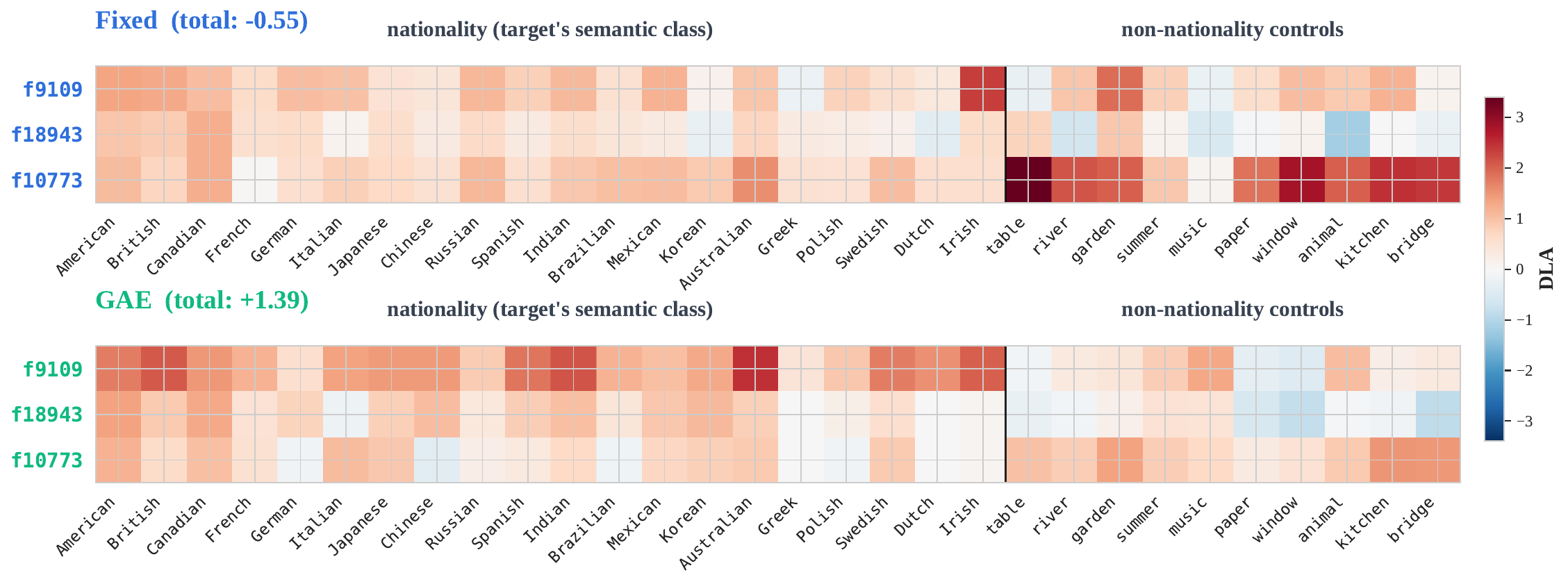}
\vspace{-2em}
\caption{\textbf{Per-feature DLA on a prompt predicting `\,American' (GPT-2, Transcoder).} Both methods share the same encoder and top-3 features; only the decoder columns differ. Each cell shows a feature's direct logit attribution (DLA) to nationality tokens (left, 20 tokens) vs.\ non-nationality controls (right, 10 tokens). Fixed's total class-specificity is $-0.55$ (circuit points away from the target class); GAE's is $+1.39$ (circuit points toward it).}
\label{fig:case_study_class_specificity}
\vspace{1em}
\end{figure}

For a prompt predicting `\,American', we compare the same feature IDs under Fixed and GAE, with identical activations.
Direct logit attribution (DLA) measures a feature's contribution to a token's logit.
For each feature, we subtract its mean DLA on 10 non-nationality controls from its mean DLA on 20 nationality tokens.
Class-specificity sums these differences over the three selected features.

Fixed scores $-0.55$, whereas GAE scores $+1.39$ (Figure~\ref{fig:case_study_class_specificity}).
For these selected features, decoder adaptation increases attribution to the target class while their activations remain unchanged.
Appendix~\ref{sec:case_study_appendix} gives the protocol and an additional first-name case study.

\vspace{-0.3em}
\subsubsection{Mechanism Analysis}
\label{subsec:ablation}

\begin{figure}[ht]
\centering
\vspace{-1.5em}
\begin{subfigure}[t]{0.48\linewidth}
\centering
\includegraphics[width=\linewidth]{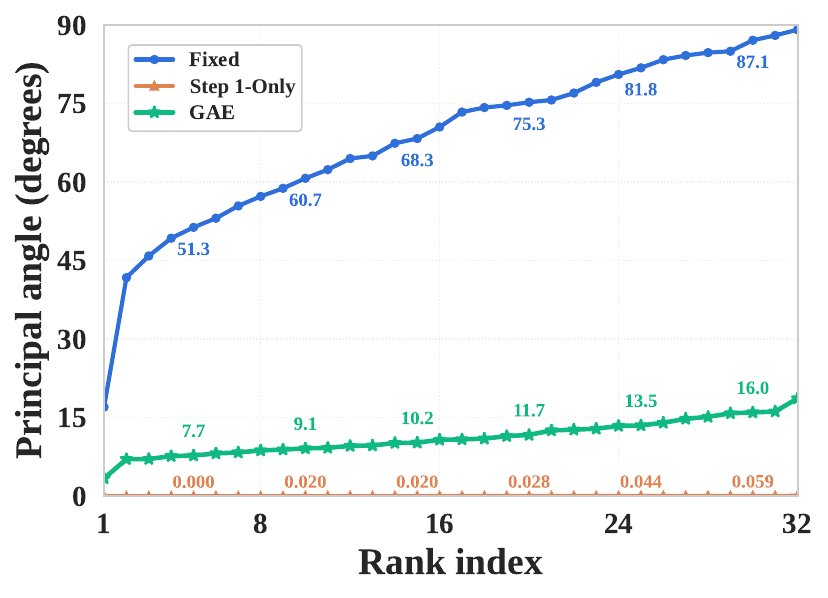}
\caption{Subspace alignment.}
\label{fig:sorted_angles}
\end{subfigure}
\hfill
\begin{subfigure}[t]{0.48\linewidth}
\centering
\includegraphics[width=\linewidth]{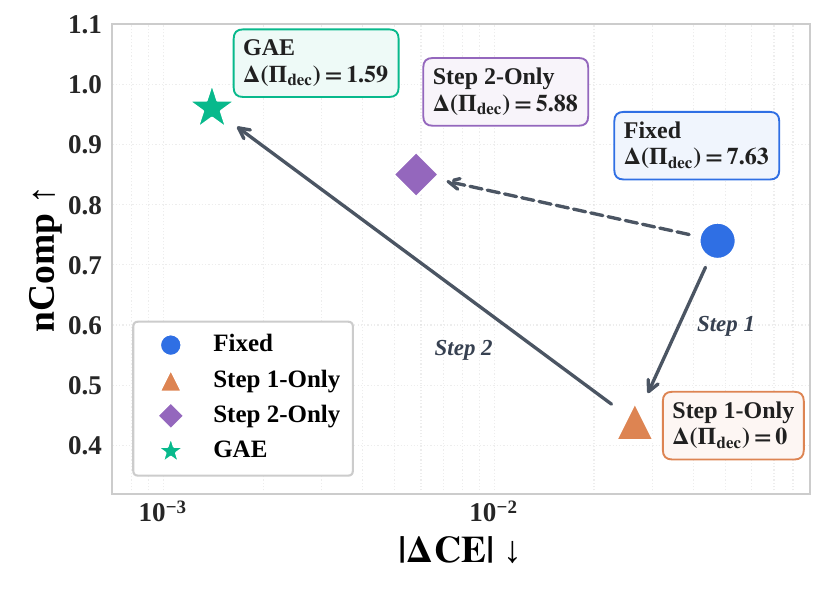}
\caption{Step ablation.}
\label{fig:step_ablation}
\end{subfigure}
\caption{\textbf{Mechanism analysis (GPT-2, Transcoder).}
\textbf{(a)}~Sorted principal angles between each explainer's top-$r$ subspace and $\widehat{\Pi}_{\mathrm{OOD}}$.
Step~1-Only yields near-zero angles, while GAE retains residual angles after decoder refitting.
\textbf{(b)}~Step ablation on HaluEval: Step~1-Only closes the faithfulness gap to 0 yet drops nComp from 0.74 to 0.44.
Step~2-Only (dashed arrow) leaves a gap of 5.88, whereas GAE achieves a gap of 1.59, higher nComp, and lower $|\Delta\mathrm{CE}|$.
This comparison supports aligning the dictionary preservation reference before refitting the decoder.}
\label{fig:geometric_analysis}
\vspace{-1.3em}
\end{figure}

\paragraph{Subspace alignment.}
\looseness=-1
Figure~\ref{fig:geometric_analysis}(a) measures the principal angles between each explainer's top-$r$ decoder subspace and the OOD-active subspace $\widehat{\Pi}_{\mathrm{OOD}}$.
Step~1-Only yields near-zero angles, while GAE retains residual angles after decoder refitting.
Both are better aligned with $\widehat{\Pi}_{\mathrm{OOD}}$ than Fixed.
Appendix~\ref{subsec:geometric_appendix} illustrates approximately linear Step~1 projection-loss improvement with $\Delta(\Pi_{\mathrm{ID}})^2$.

\vspace{-0.5em}
\paragraph{Step ablation.}
Figure~\ref{fig:geometric_analysis}(b) ablates each step of GAE on HaluEval.
Step~2-Only skips Step~1 and uses the Fixed ID dictionary as the dictionary preservation reference.
The encoder, data, objective, hyperparameters, and evaluation protocol are unchanged.
Step~1-Only closes the faithfulness gap from 7.63 to 0 by construction, yet nComp drops from 0.74 to 0.44, showing that subspace alignment alone does not ensure higher nComp.
Step~2-Only improves nComp to 0.85 and reduces $|\Delta\mathrm{CE}|$ to 0.0058, but leaves a faithfulness gap of 5.88.
In contrast, GAE achieves a smaller gap (1.59), the highest nComp (0.96), and the lowest $|\Delta\mathrm{CE}|$ (0.0014).
These results support complementary roles: Step~1 provides an OOD-aligned dictionary as the preservation reference, and Step~2 improves its compatibility with the frozen sparse feature activations.
Corollary~\ref{cor:gae_projection_bound} gives the corresponding projection-loss bound for GAE.
See Appendix~\ref{subsec:hp_sensitivity} for sensitivity.
\vspace{-0.5em}

\vspace{-0.3em}
\section{Conclusion}
\label{sec:conclusion}
\vspace{-0.7em}
We showed that OOD faithfulness degradation in dictionary-based explainers has a geometric cause: the decoder subspace drifts from the directions the model actively uses.
The faithfulness gap $\Delta(\Pi_{\mathrm{dec}})$ formalizes this misalignment and provably controls the reducible part of OOD projection loss.
GAE targets the gap with a closed-form subspace rotation and decoder refit, using only unlabeled OOD activations while preserving the original sparse feature activations.
Across three shift types on GPT-2 Small and Pythia-1.4B, GAE outperforms all training-based baselines on 5 of 6 settings, completing in under 3 seconds without any gradient computation.
An additional evaluation on Llama-3.1-8B with a public 128K-feature SAE improves all three reported metrics on HaluEval, with adaptation taking approximately 30 seconds.
GAE also relies on a top-$r$ SVD truncation of $W_{\mathrm{dec}}^{\mathrm{ID}}$, so any feature information carried in the residual $(d{-}r)$ singular directions is dropped before adaptation.
Extending GAE to adaptive rank selection, encoder adaptation, and connections to optimal transport on the Grassmannian~\cite{edelman1998geometry} are promising future directions.

\clearpage
\section*{Acknowledgements}
This work was supported by the National Research Foundation of Korea(NRF) grant funded by the Korea government(MSIT)(RS-2024-00457216).
This research was supported by a fund from the research program of the Korea Medical Institute.(KMI20250213)
This work was supported by National Institute for Mathematical Sciences (NIMS) grant funded by the Korea government (MSIT) (No. S26120000).

\bibliographystyle{unsrtnat}
\bibliography{main.bib}

\clearpage
\appendix

\section{Proofs and Derivations}
\label{sec:proofs_and_derivation}


\subsection{Proof of Proposition~\ref{prop:second-moment_shift_enlarges_the_faithfulness_gap}}
\label{subsec:proof_prop2}

\paragraph{Setup.}
Write the second-moment shift as $E = M_{\mathrm{OOD}} - M_{\mathrm{ID}}$, so that $M_{\mathrm{OOD}} = M_{\mathrm{ID}} + E$. The projectors $\Pi_{\mathrm{ID}}$ and $\Pi_{\mathrm{OOD}}$ correspond to the top-$r$ eigenspaces of $M_{\mathrm{ID}}$ and $M_{\mathrm{ID}} + E$, respectively. We wish to bound $\Delta(\Pi_{\mathrm{ID}}) = \|\Pi_{\mathrm{OOD}} - \Pi_{\mathrm{ID}}\|_F$ in terms of $\|E\|_F$.

\paragraph{Step 1: Projector distance via principal angles.}
Let $\theta_1, \ldots, \theta_r$ be the principal angles between the column spaces of $\Pi_{\mathrm{ID}}$ and $\Pi_{\mathrm{OOD}}$. For two rank-$r$ orthogonal projectors, the Frobenius norm of their difference satisfies
\[
\|\Pi_{\mathrm{OOD}} - \Pi_{\mathrm{ID}}\|_F^2 = 2\sum_{i=1}^r \sin^2\theta_i.
\]

\paragraph{Step 2: Using the ID eigengap.}
The eigenbasis argument in Proposition~\ref{prop:faithfulness_gap_controls_excess_hidden-space_faithfulness}, applied to $M_{\mathrm{ID}}$, gives
\[
\frac{\gamma_{\mathrm{ID}}}{2}\|\Pi_{\mathrm{ID}}-\Pi_{\mathrm{OOD}}\|_F^2
\le \operatorname{tr}((\Pi_{\mathrm{ID}}-\Pi_{\mathrm{OOD}})M_{\mathrm{ID}}),
\]
where $\gamma_{\mathrm{ID}} = \lambda_r(M_{\mathrm{ID}}) - \lambda_{r+1}(M_{\mathrm{ID}})$ is the eigengap of $M_{\mathrm{ID}}$ at rank $r$.

\paragraph{Step 3: Using OOD optimality.}
Since $\Pi_{\mathrm{OOD}}$ maximizes the retained OOD energy, $\operatorname{tr}((\Pi_{\mathrm{ID}}-\Pi_{\mathrm{OOD}})M_{\mathrm{OOD}})\le0$.
Using $M_{\mathrm{ID}}=M_{\mathrm{OOD}}-E$ and Frobenius Cauchy--Schwarz,
\begin{align*}
\operatorname{tr}((\Pi_{\mathrm{ID}}-\Pi_{\mathrm{OOD}})M_{\mathrm{ID}})
&\le -\operatorname{tr}((\Pi_{\mathrm{ID}}-\Pi_{\mathrm{OOD}})E)\\
&\le \|\Pi_{\mathrm{ID}}-\Pi_{\mathrm{OOD}}\|_F\|E\|_F.
\end{align*}

\paragraph{Step 4: Combining.}
Combining Steps~2 and 3 for nonzero projector distance,
\[
\Delta(\Pi_{\mathrm{ID}})
= \|\Pi_{\mathrm{OOD}} - \Pi_{\mathrm{ID}}\|_F
= \sqrt{2\sum_{i=1}^r \sin^2\theta_i}
\le \frac{2}{\gamma_{\mathrm{ID}}}\,\|E\|_F
= \frac{2}{\gamma_{\mathrm{ID}}}\,\|M_{\mathrm{OOD}} - M_{\mathrm{ID}}\|_F.
\]
If the projector distance is zero, the bound is immediate.
\qed

\subsection{Decomposition Equation~\texorpdfstring{\eqref{eq:decomposition}}{(decomposition)}}
\label{subsec:proof_decomposition}

\paragraph{Step 1: Expressing $\mathcal{L}_{\mathrm{OOD}}$ in terms of $M_{\mathrm{OOD}}$.}
For any rank-$r$ orthogonal $\Pi_{\mathrm{dec}}$ $\in \mathcal{C}_r$, the reconstruction error under $P_{\mathrm{OOD}}$ is
\begin{align}
\mathcal{L}_{\mathrm{OOD}}(\Pi_{\mathrm{dec}})
&= \mathbb{E}_{X \sim P_{\mathrm{OOD}}}\|h(X) - \Pi_{\mathrm{dec}}\,h(X)\|_2^2
\nonumber\\
&= \mathbb{E}\|(I - \Pi_{\mathrm{dec}})\,h(X)\|_2^2
\nonumber\\
&= \mathbb{E}\,\operatorname{tr}\!\big[(I - \Pi_{\mathrm{dec}})\,h(X)\,h(X)^\top\,(I - \Pi_{\mathrm{dec}})^\top\big]
\nonumber\\
&= \operatorname{tr}\!\big[(I - \Pi_{\mathrm{dec}})\,M_{\mathrm{OOD}}\,(I - \Pi_{\mathrm{dec}})\big],
\label{eq:loss_trace}
\end{align}
where the third step uses $\|a\|_2^2 = \operatorname{tr}(aa^\top)$ and the fourth step swaps expectation and trace, with $M_{\mathrm{OOD}} = \mathbb{E}[h(X)h(X)^\top]$. This derivation explains why we use the uncentered second moment rather than the centered covariance: the projection objective depends directly on $M_{\mathrm{OOD}}$, which captures covariance shift as well as mean shift. Since $I - \Pi_{\mathrm{dec}}$ is an orthogonal projector, it is idempotent: $(I-\Pi_{\mathrm{dec}})^2 = I - \Pi_{\mathrm{dec}}$. Using the cyclic property of trace,
\begin{align}
\operatorname{tr}\!\big[(I - \Pi_{\mathrm{dec}})\,M_{\mathrm{OOD}}\,(I - \Pi_{\mathrm{dec}})\big]
= \operatorname{tr}\!\big[(I - \Pi_{\mathrm{dec}})^2\,M_{\mathrm{OOD}}\big]
= \operatorname{tr}\!\big[(I - \Pi_{\mathrm{dec}})\,M_{\mathrm{OOD}}\big].
\end{align}

\paragraph{Step 2: Deriving the decomposition.}
Applying the same identity with $\Pi_{\mathrm{OOD}}$ gives $\mathcal{L}_{\mathrm{OOD}}(\Pi_{\mathrm{OOD}}) = \operatorname{tr}\!\big[(I - \Pi_{\mathrm{OOD}})\,M_{\mathrm{OOD}}\big]$. Taking the difference,
\begin{align}
\mathcal{L}_{\mathrm{OOD}}(\Pi_{\mathrm{dec}}) - \mathcal{L}_{\mathrm{OOD}}(\Pi_{\mathrm{OOD}})
&= \operatorname{tr}\!\big[(I - \Pi_{\mathrm{dec}})\,M_{\mathrm{OOD}}\big] - \operatorname{tr}\!\big[(I - \Pi_{\mathrm{OOD}})\,M_{\mathrm{OOD}}\big]
\nonumber\\
&= \operatorname{tr}\!\big[(\Pi_{\mathrm{OOD}} - \Pi_{\mathrm{dec}})\,M_{\mathrm{OOD}}\big].
\end{align}
Rearranging gives the claimed decomposition:
\[
\mathcal{L}_{\mathrm{OOD}}(\Pi_{\mathrm{dec}})
= \mathcal{L}_{\mathrm{OOD}}(\Pi_{\mathrm{OOD}})
+ \operatorname{tr}\!\big[(\Pi_{\mathrm{OOD}} - \Pi_{\mathrm{dec}})\,M_{\mathrm{OOD}}\big].
\]

\paragraph{Step 3: Nonnegativity and optimality.}
It remains to show that the second term is nonneg\-ative. Let $M_{\mathrm{OOD}} = \sum_{i=1}^d \lambda_i\,u_i u_i^\top$ be the eigendecomposition with $\lambda_1 \ge \cdots \ge \lambda_d \ge 0$. Since $\Pi_{\mathrm{OOD}}$ projects onto the top-$r$ eigenspace,
\[
\operatorname{tr}(\Pi_{\mathrm{OOD}}\,M_{\mathrm{OOD}}) = \sum_{i=1}^r \lambda_i.
\]
By Ky Fan's maximum principle, $\sum_{i=1}^r \lambda_i = \max_{\Pi_{\mathrm{dec}} \in \mathcal{C}_r} \operatorname{tr}(\Pi_{\mathrm{dec}}\,M_{\mathrm{OOD}})$. Therefore, for any $\Pi_{\mathrm{dec}} \in \mathcal{C}_r$,
\[
\operatorname{tr}\!\big[(\Pi_{\mathrm{OOD}} - \Pi_{\mathrm{dec}})\,M_{\mathrm{OOD}}\big]
= \operatorname{tr}(\Pi_{\mathrm{OOD}}\,M_{\mathrm{OOD}}) - \operatorname{tr}(\Pi_{\mathrm{dec}}\,M_{\mathrm{OOD}})
\ge 0,
\]
with equality if and only if $\Pi_{\mathrm{dec}}$ also projects onto a top-$r$ eigenspace of $M_{\mathrm{OOD}}$. This proves both the nonnegativity and the optimality $\Pi_{\mathrm{OOD}} \in \arg\min_{\Pi_{\mathrm{dec}} \in \mathcal{C}_r}\mathcal{L}_{\mathrm{OOD}}(\Pi_{\mathrm{dec}})$.
\qed

\subsection{Proof of Proposition~\ref{prop:faithfulness_gap_controls_excess_hidden-space_faithfulness}}
\label{subsec:proof_prop1}

From the decomposition~\eqref{eq:decomposition}, the explainer-dependent component equals $\operatorname{tr}\!\big[(\Pi_{\mathrm{OOD}} - \Pi_{\mathrm{dec}})\,M_{\mathrm{OOD}}\big]$. We derive both bounds by expanding this trace in the eigenbasis of $M_{\mathrm{OOD}}$.

\paragraph{Setup.}
Let $M_{\mathrm{OOD}} = \sum_{i=1}^d \lambda_i\,u_i u_i^\top$ be the eigendecomposition with $\lambda_1 \ge \cdots \ge \lambda_d \ge 0$. The OOD-active subspace is $\Pi_{\mathrm{OOD}} = \sum_{i=1}^r u_i u_i^\top$, so $u_i^\top \Pi_{\mathrm{OOD}}\,u_i = \mathbf{1}_{i \le r}$. For the explainer subspace $\Pi_{\mathrm{dec}} \in \mathcal{C}_r$, define $p_i := u_i^\top \Pi_{\mathrm{dec}}\, u_i \in [0,1]$, the fraction of the $i$-th OOD eigendirection captured by $\Pi_{\mathrm{dec}}$. Expanding the trace in this eigenbasis gives
\begin{align}
\operatorname{tr}\!\big[(\Pi_{\mathrm{OOD}} - \Pi_{\mathrm{dec}})\,M_{\mathrm{OOD}}\big]
&= \sum_{i=1}^d \lambda_i\,(u_i^\top \Pi_{\mathrm{OOD}}\,u_i - p_i)
\nonumber\\
&= \sum_{i=1}^r \lambda_i\,(1 - p_i) - \sum_{i=r+1}^d \lambda_i\,p_i.
\label{eq:trace_expansion}
\end{align}
The first sum captures the OOD energy in the top-$r$ directions that the explainer misses (since $1 - p_i$ is the fraction lost). The second sum captures the OOD energy in the bottom directions that the explainer unnecessarily covers.

\paragraph{Connecting to the faithfulness gap.}
For two rank-$r$ subspaces, the Frobenius norm of their difference satisfies $\|\Pi_{\mathrm{OOD}} - \Pi_{\mathrm{dec}}\|_F^2 = 2\sum_{i=1}^r (1 - p_i)$. Moreover, since both have the same rank, $\sum_{i=1}^r (1-p_i) = \sum_{i=r+1}^d p_i$. Denoting $S := \sum_{i=1}^r (1 - p_i)$, we have
\begin{equation}
\Delta(\Pi_{\mathrm{dec}})^2 = \|\Pi_{\mathrm{OOD}} - \Pi_{\mathrm{dec}}\|_F^2 = 2S, \quad\text{and}\quad \sum_{i=r+1}^d p_i = S.
\label{eq:S_identity}
\end{equation}

\paragraph{Upper bound.}
Using $\lambda_i \le \lambda_1$ for $i \le r$ and $\lambda_i \ge \lambda_d$ for $i > r$ in Equation~\eqref{eq:trace_expansion},
\begin{align}
\sum_{i=1}^r \lambda_i(1-p_i) - \sum_{i=r+1}^d \lambda_i\,p_i
&\le \lambda_1 \sum_{i=1}^r (1-p_i) - \lambda_d \sum_{i=r+1}^d p_i
\nonumber\\
&= \lambda_1\,S - \lambda_d\,S
= (\lambda_1 - \lambda_d)\,S
= \frac{\lambda_1(M_{\mathrm{OOD}}) - \lambda_d(M_{\mathrm{OOD}})}{2}\,\Delta(\Pi_{\mathrm{dec}})^2.
\end{align}

\paragraph{Lower bound.}
Using $\lambda_i \ge \lambda_r$ for $i \le r$ and $\lambda_i \le \lambda_{r+1}$ for $i > r$,
\begin{align}
\sum_{i=1}^r \lambda_i(1-p_i) - \sum_{i=r+1}^d \lambda_i\,p_i
&\ge \lambda_r \sum_{i=1}^r (1-p_i) - \lambda_{r+1} \sum_{i=r+1}^d p_i
\nonumber\\
&= \lambda_r\,S - \lambda_{r+1}\,S
= (\lambda_r - \lambda_{r+1})\,S
= \frac{\gamma_{\mathrm{OOD}}}{2}\,\Delta(\Pi_{\mathrm{dec}})^2.
\end{align}
\qed

For $\gamma_{\mathrm{ID}}>0$, substituting Proposition~\ref{prop:second-moment_shift_enlarges_the_faithfulness_gap} into the upper bound of Proposition~\ref{prop:faithfulness_gap_controls_excess_hidden-space_faithfulness} gives
\[
\mathcal{L}_{\mathrm{OOD}}(\Pi_{\mathrm{ID}})
-
\mathcal{L}_{\mathrm{OOD}}(\Pi_{\mathrm{OOD}})
\le
\frac{\lambda_1(M_{\mathrm{OOD}})-\lambda_d(M_{\mathrm{OOD}})}{2}
\left(
\frac{2}{\gamma_{\mathrm{ID}}}\,
\|M_{\mathrm{OOD}}-M_{\mathrm{ID}}\|_F
\right)^2.
\]
This bound is quadratic in the second-moment shift and does not require $\gamma_{\mathrm{OOD}}>0$.

\subsection{Derivation of the GAE Procrustes Solution (Step~1)}
\label{subsec:derivation_gae_procrustes}

We derive the closed-form solution to Equation~\eqref{eq:gae_obj}.
Substituting $\widetilde{W}_{\mathrm{dec}}(T) = U_{\mathrm{OOD}}^{(:r)}\, T\, U_{\mathrm{dec}}^\top\, W_{\mathrm{dec}}^{\mathrm{ID}}$ and expanding:
\begin{align}
\|\widetilde{W}_{\mathrm{dec}}(T) - W_{\mathrm{dec}}^{\mathrm{ID}}\|_F^2
&= \|W_{\mathrm{dec}}^{\mathrm{ID}}\|_F^2 + \|U_{\mathrm{dec}}^\top W_{\mathrm{dec}}^{\mathrm{ID}}\|_F^2 - 2\operatorname{tr}\!\big[(W_{\mathrm{dec}}^{\mathrm{ID}})^\top U_{\mathrm{OOD}}^{(:r)}\, T\, U_{\mathrm{dec}}^\top\, W_{\mathrm{dec}}^{\mathrm{ID}}\big],
\end{align}
where we used $T^\top T = I$ and $(U_{\mathrm{OOD}}^{(:r)})^\top U_{\mathrm{OOD}}^{(:r)} = I$. The first two terms are independent of $T$, so minimization reduces to
\[
T^\star = \arg\max_{T \in \mathcal{O}_r} \operatorname{tr}(G\, T), \qquad G = U_{\mathrm{dec}}^\top\, W_{\mathrm{dec}}^{\mathrm{ID}}\, (W_{\mathrm{dec}}^{\mathrm{ID}})^\top\, U_{\mathrm{OOD}}^{(:r)}.
\]
Let $G = \widetilde{U}\Sigma\widetilde{V}^\top$ be the SVD of $G$. Setting $R = \widetilde{V}^\top T \widetilde{U}$, we have $\operatorname{tr}(G\,T) = \operatorname{tr}(\Sigma\, R) = \sum_i \sigma_i R_{ii}$. Since $R \in \mathcal{O}_r$, each $|R_{ii}| \le 1$, so the maximum is achieved at $R = I_r$, giving $T^\star = \widetilde{V}\widetilde{U}^\top$.
\qed

The same orthonormality identities give the Step~1 Gram identity
\begin{equation}
\label{eq:gae_retained_gram}
\widetilde{W}_{\mathrm{dec}}(T)^\top\widetilde{W}_{\mathrm{dec}}(T)
= (W_{\mathrm{dec}}^{\mathrm{ID}})^\top U_{\mathrm{dec}}U_{\mathrm{dec}}^\top W_{\mathrm{dec}}^{\mathrm{ID}}.
\end{equation}
Thus Step~1 preserves the pairwise inner products of the ID dictionary columns projected onto $\operatorname{span}(U_{\mathrm{dec}})$.
Step~2 and interpolation can change this rank-$r$ Gram geometry.

\subsection{Derivation of the Step~2 Closed-Form Solution}
\label{subsec:derivation_gae_step2}

Given paired target activations $h_i$ and fixed sparse feature activations $z_i$, Step~2 solves the following convex quadratic objective over $W_{\mathrm{dec}} \in \mathbb{R}^{d \times k}$ and $b \in \mathbb{R}^d$:
\[
\min_{W_{\mathrm{dec}},\, b}\;
\frac{1}{N}\sum_{i=1}^N \|h_i - W_{\mathrm{dec}}\,z_i - b\|^2
\;+\;
\lambda_{\mathrm{geom}}\,\|(I - \widehat{\Pi}_{\mathrm{OOD}})\,W_{\mathrm{dec}}\|_F^2
\;+\;
\lambda_{\mathrm{pres}}\,\|W_{\mathrm{dec}} - \widetilde{W}_{\mathrm{dec}}(T^\star)\|_F^2.
\]

\paragraph{Bias.}
Setting $\partial / \partial b = 0$ gives $b^\star = \frac{1}{N}\sum_i h_i - W_{\mathrm{dec}}^\star \frac{1}{N}\sum_i z_i$, i.e., Equation~\eqref{eq:gae_step2_bias}.

\paragraph{Decoder.}
Substituting $b^\star$ centers the data: define $h_i^c = h_i - \frac{1}{N}\sum_j h_j$ and $z_i^c = z_i - \frac{1}{N}\sum_j z_j$. The problem reduces to
\[
\min_{W_{\mathrm{dec}}}\;
\frac{1}{N}\sum_{i=1}^N \|h_i^c - W_{\mathrm{dec}}\,z_i^c\|^2
\;+\;
\lambda_{\mathrm{geom}}\,\|(I - \widehat{\Pi}_{\mathrm{OOD}})\,W_{\mathrm{dec}}\|_F^2
\;+\;
\lambda_{\mathrm{pres}}\,\|W_{\mathrm{dec}} - \widetilde{W}_{\mathrm{dec}}(T^\star)\|_F^2.
\]
Setting $\partial / \partial W_{\mathrm{dec}} = 0$:
\[
\bigl[\lambda_{\mathrm{pres}}\,I + \lambda_{\mathrm{geom}}(I - \widehat{\Pi}_{\mathrm{OOD}})\bigr]\,W_{\mathrm{dec}}
\;+\;
W_{\mathrm{dec}}\;\tfrac{1}{N}\textstyle\sum_i z_i^c z_i^{c\top}
\;=\;
\tfrac{1}{N}\textstyle\sum_i h_i^c z_i^{c\top}
\;+\;
\lambda_{\mathrm{pres}}\,\widetilde{W}_{\mathrm{dec}}(T^\star).
\]
The left-hand coefficient $\Lambda := \lambda_{\mathrm{pres}}\,I + \lambda_{\mathrm{geom}}(I - \widehat{\Pi}_{\mathrm{OOD}})$ is diagonal in the OOD basis, with eigenvalue $\lambda_{\mathrm{pres}}$ on $\mathrm{span}(U_{\mathrm{OOD}}^{(:r)})$ and $\lambda_{\mathrm{pres}} + \lambda_{\mathrm{geom}}$ on its complement. Denoting $B = \frac{1}{N}\sum_i z_i^c z_i^{c\top} + \lambda_{\mathrm{pres}}\,I$ and $C = \frac{1}{N}\sum_i h_i^c z_i^{c\top} + \lambda_{\mathrm{pres}}\,\widetilde{W}_{\mathrm{dec}}(T^\star)$, the system decouples row-wise in the OOD basis into two standard ridge regressions:
\begin{itemize}[nosep,leftmargin=*]
\item \textbf{Inside $\widehat{\Pi}_{\mathrm{OOD}}$} (first $r$ rows): $W_{\mathrm{in}} = C_{\mathrm{in}}\, B^{-1}$, with ridge level $\lambda_{\mathrm{pres}}$.
\item \textbf{Outside $\widehat{\Pi}_{\mathrm{OOD}}$} (remaining $d{-}r$ rows): $W_{\mathrm{out}} = C_{\mathrm{out}}\,(B + \lambda_{\mathrm{geom}}\,I)^{-1}$, with ridge level $\lambda_{\mathrm{pres}} + \lambda_{\mathrm{geom}}$.
\end{itemize}
Combining via the projector yields Equation~\eqref{eq:gae_step2_dec}:
$W_{\mathrm{dec}}^{\mathrm{GAE}} = \widehat{\Pi}_{\mathrm{OOD}}\,C\,B^{-1} + (I - \widehat{\Pi}_{\mathrm{OOD}})\,C\,(B + \lambda_{\mathrm{geom}}\,I)^{-1}$.
The centered feature covariance has rank at most $N-1$.
For $N\le k$, $\lambda_{\mathrm{pres}}>0$ is therefore required for $B$ to be invertible and ensures that both inverse matrices exist.
\qed

\subsection{Proof of Theorem~\ref{thm:improvement_over_id}}
\label{subsec:proof_theorem_improvement}

\begin{proof}
By Equation~\eqref{eq:gae_param}, the rank-$r$ dictionary $\widetilde{W}_{\mathrm{dec}}(T^\star)$ has column space $\mathrm{span}(U_{\mathrm{OOD}}^{(:r)})$, with projector
\begin{equation}
\widehat{\Pi}_{\mathrm{OOD}} = U_{\mathrm{OOD}}^{(:r)}(U_{\mathrm{OOD}}^{(:r)})^\top.
\end{equation}
Under the condition $\widehat{\Pi}_{\mathrm{OOD}} = \Pi_{\mathrm{OOD}}$, this gives $\Delta(\widehat{\Pi}_{\mathrm{OOD}}) = \|\Pi_{\mathrm{OOD}} - \widehat{\Pi}_{\mathrm{OOD}}\|_F = 0$ and therefore
\begin{equation}
\mathcal{L}_{\mathrm{OOD}}(\widehat{\Pi}_{\mathrm{OOD}}) = \mathcal{L}_{\mathrm{OOD}}(\Pi_{\mathrm{OOD}}).
\end{equation}
Applying the lower bound of Proposition~\ref{prop:faithfulness_gap_controls_excess_hidden-space_faithfulness} to $\Pi_{\mathrm{dec}} = \Pi_{\mathrm{ID}}$:
\begin{equation}
\mathcal{L}_{\mathrm{OOD}}(\Pi_{\mathrm{ID}}) - \mathcal{L}_{\mathrm{OOD}}(\Pi_{\mathrm{OOD}})
\;\ge\;
\frac{\gamma_{\mathrm{OOD}}}{2}\,\Delta(\Pi_{\mathrm{ID}})^2.
\end{equation}
Substituting $\mathcal{L}_{\mathrm{OOD}}(\widehat{\Pi}_{\mathrm{OOD}}) = \mathcal{L}_{\mathrm{OOD}}(\Pi_{\mathrm{OOD}})$ and rearranging:
\begin{equation}
\mathcal{L}_{\mathrm{OOD}}(\widehat{\Pi}_{\mathrm{OOD}})
= \mathcal{L}_{\mathrm{OOD}}(\Pi_{\mathrm{OOD}})
\le \mathcal{L}_{\mathrm{OOD}}(\Pi_{\mathrm{ID}})
- \frac{\gamma_{\mathrm{OOD}}}{2}\,\Delta(\Pi_{\mathrm{ID}})^2.
\end{equation}
\end{proof}

\subsection{Proof of Corollary~\ref{cor:gae_projection_bound}}
\label{subsec:proof_gae_projection_bound}

\begin{proof}
Step~1 gives $(I-\widehat{\Pi}_{\mathrm{OOD}})\widetilde{W}_{\mathrm{dec}}(T^\star)=0$.
Left-multiplying Equation~\eqref{eq:gae_step2_dec} by $I-\widehat{\Pi}_{\mathrm{OOD}}$ therefore gives
\begin{equation}
(I-\widehat{\Pi}_{\mathrm{OOD}})W_{\mathrm{dec}}^{\mathrm{GAE}}
=R_\perp(B+\lambda_{\mathrm{geom}}I)^{-1}.
\end{equation}
Since $B-\lambda_{\mathrm{pres}}I$ is a centered sample covariance, $B+\lambda_{\mathrm{geom}}I\succeq(\lambda_{\mathrm{pres}}+\lambda_{\mathrm{geom}})I$.
Applying Equation~\eqref{eq:decoder_interpolation} yields
\begin{equation}
\label{eq:gae_off_subspace_bound}
\|(I-\widehat{\Pi}_{\mathrm{OOD}})W_{\mathrm{dec}}\|_F
\le
\frac{(1-\alpha)\|R_\perp\|_F}{\lambda_{\mathrm{pres}}+\lambda_{\mathrm{geom}}}.
\end{equation}

Let $U_r\Sigma_rV_r^\top$ be a top-$r$ SVD truncation of $W_{\mathrm{dec}}$, so $\Pi_{\mathrm{dec}}=U_rU_r^\top$.
The matrix inequality $W_{\mathrm{dec}}^\top W_{\mathrm{dec}}\succeq W_{\mathrm{dec}}^\top\widehat{\Pi}_{\mathrm{OOD}}W_{\mathrm{dec}}$ implies $\sigma_r(W_{\mathrm{dec}})\ge\sigma_r(\widehat{\Pi}_{\mathrm{OOD}}W_{\mathrm{dec}})>0$.
Using $U_r=W_{\mathrm{dec}}V_r\Sigma_r^{-1}$, the equal-rank projector identity gives
\begin{align}
\|\Pi_{\mathrm{dec}}-\widehat{\Pi}_{\mathrm{OOD}}\|_F
&=\sqrt{2}\|(I-\widehat{\Pi}_{\mathrm{OOD}})U_r\|_F\nonumber\\
&\le\frac{\sqrt{2}\|(I-\widehat{\Pi}_{\mathrm{OOD}})W_{\mathrm{dec}}\|_F}{\sigma_r(W_{\mathrm{dec}})}\nonumber\\
&\le\frac{2\|(I-\widehat{\Pi}_{\mathrm{OOD}})W_{\mathrm{dec}}\|_F}{\sigma_r(\widehat{\Pi}_{\mathrm{OOD}}W_{\mathrm{dec}})}.
\label{eq:gae_projector_distance_bound}
\end{align}
The triangle inequality gives $\Delta(\Pi_{\mathrm{dec}})\le\varepsilon_N+\|\Pi_{\mathrm{dec}}-\widehat{\Pi}_{\mathrm{OOD}}\|_F$.
Combining this with Equations~\eqref{eq:gae_off_subspace_bound} and~\eqref{eq:gae_projector_distance_bound}, then applying the upper bound in Proposition~\ref{prop:faithfulness_gap_controls_excess_hidden-space_faithfulness}, proves Equation~\eqref{eq:gae_projection_bound}.

At $\alpha=1$, the rank condition on $W_{\mathrm{dec}}^{\mathrm{ID}}$ makes the Step~1-Only projector equal to $\widehat{\Pi}_{\mathrm{OOD}}$.
Its faithfulness gap is $\varepsilon_N$, so Proposition~\ref{prop:faithfulness_gap_controls_excess_hidden-space_faithfulness} gives excess projection loss at most $\beta_{\mathrm{OOD}}\varepsilon_N^2/2$.
This endpoint applies when Step~2 is skipped, including zero Step~2 regularization weights.
\end{proof}

\subsection{\texorpdfstring{Projection-Level Surrogate and Downstream Causal Metrics}{Projection-Level Surrogate and Downstream Causal Metrics}}
\label{subsec:projection_surrogate_downstream_metrics}

This subsection relates projection loss to dictionary reconstruction and feature-intervention scores.
The metrics are not equivalent to $\mathcal{L}_{\mathrm{OOD}}$.
We bound activation-replacement loss with an additional reconstruction-error term and feature-score changes under stability assumptions.

\paragraph{Activation replacement.}
The dictionary reconstruction $\hat{h}(X)=W_{\mathrm{dec}}z(X)+b_{\mathrm{dec}}$ need not equal $\Pi_{\mathrm{dec}}h(X)$.
Write $\Delta\mathrm{CE}(X):=\mathrm{CE}(\hat{h}(X))-\mathrm{CE}(h(X))$, matching Equation~\eqref{eq:delta_ce}.

\begin{assumption}[Local smoothness]
\label{assump:local_smoothness}
There exists $K_{\mathrm{CE}}<\infty$ such that $|\mathrm{CE}(\hat{h}(X))-\mathrm{CE}(h(X))|\le K_{\mathrm{CE}}\|h(X)-\hat{h}(X)\|_2$ for OOD activations and their dictionary reconstructions.
\end{assumption}

\begin{proposition}[Activation-replacement loss bound]
\label{prop:delta_ce_projection_bound}
Under Assumption~\ref{assump:local_smoothness}, for $X\sim P_{\mathrm{OOD}}$,
\[
\mathbb{E}|\Delta\mathrm{CE}(X)|
\le K_{\mathrm{CE}}\left(
\sqrt{\mathcal{L}_{\mathrm{OOD}}(\Pi_{\mathrm{dec}})}
+\sqrt{\mathbb{E}\|\Pi_{\mathrm{dec}}h(X)-\hat{h}(X)\|_2^2}
\right).
\]
\end{proposition}

The bound combines projection error with the additional error between the projection and the actual sparse reconstruction.
Small projection error and small additional reconstruction error together yield a small upper bound on activation-replacement loss.

\begin{proof}
\begingroup
\begin{align}
\mathbb{E}|\Delta\mathrm{CE}(X)|
&\le K_{\mathrm{CE}}\,\mathbb{E}\|h-\hat{h}\|_2\nonumber\\
&\le K_{\mathrm{CE}}\left(\mathbb{E}\|h-\Pi_{\mathrm{dec}}h\|_2+\mathbb{E}\|\Pi_{\mathrm{dec}}h-\hat{h}\|_2\right)\nonumber\\
&\le K_{\mathrm{CE}}\left(\sqrt{\mathcal{L}_{\mathrm{OOD}}(\Pi_{\mathrm{dec}})}+\sqrt{\mathbb{E}\|\Pi_{\mathrm{dec}}h-\hat{h}\|_2^2}\right).
\end{align}
\endgroup
\end{proof}

\paragraph{Feature interventions.}
Feature-intervention metrics require one more condition.
Unlike $\Delta$CE, nComp and nAOPC depend on which features are selected and ablated.
Preserving feature activations does not ensure unchanged causal rankings after decoder adaptation.
The following bound assumes stability of the selected feature sets and their intervention effects.
Let $a_m:=\frac{\ell_0-\ell_{\setminus m}}{|\ell_0-\ell_\varnothing|}$ denote the budget-$m$ normalized logit drop in Equation~\eqref{eq:naopc} for the Fixed ID decoder $(W_{\mathrm{dec}}^{\mathrm{ID}},b_{\mathrm{dec}}^{\mathrm{ID}})$, and let $a_m^{\mathrm{GAE}}$ denote the same quantity for the deployed GAE decoder $(W_{\mathrm{dec}},b_{\mathrm{dec}})$.
For the same masked feature vector $v$, the reconstruction difference is $\delta(v):=(W_{\mathrm{dec}}-W_{\mathrm{dec}}^{\mathrm{ID}})v+(b_{\mathrm{dec}}-b_{\mathrm{dec}}^{\mathrm{ID}})$.
Write $z_{\setminus m}$ for $z$ with the selected $m$ feature activations set to zero.

\begin{assumption}[Feature stability]
\label{assump:feature_stability}
For each feature budget used by nComp or nAOPC, both decoders use the same selected feature set and target token.
The corresponding normalized logit-drop score is $K_m$-Lipschitz in the full, ablated, and empty replacement activations along this comparison, using the sum of their Euclidean distances.
Both decoders use the same empty-reconstruction convention, either retaining their respective biases or setting the replacement activation to zero.
The normalization denominators remain bounded away from zero.
\end{assumption}

\begin{proposition}[Feature-score stability under decoder adaptation]
\label{prop:feature_metric_transfer}
Under Assumption~\ref{assump:feature_stability}, for fixed decoders and $X\sim P_{\mathrm{OOD}}$,
\begingroup
\[
\mathbb{E}\bigl|a_m^{\mathrm{GAE}}-a_m\bigr|
\le
K_m\left(
2\|W_{\mathrm{dec}}-W_{\mathrm{dec}}^{\mathrm{ID}}\|_F\,\mathbb{E}\|z\|_2
+3\|b_{\mathrm{dec}}-b_{\mathrm{dec}}^{\mathrm{ID}}\|_2
\right).
\]
\endgroup
Averaging over budgets gives a bound for nAOPC with the average of $K_m$, and taking $m=m^*$ gives the nComp bound.
\end{proposition}

\begin{proof}
The full and ablated reconstruction changes are $\delta(z)$ and $\delta(z_{\setminus m})$.
The empty-reconstruction change is $\delta(0)$ when biases are retained and zero under the common zero-activation convention.
The Lipschitz assumption therefore gives
\begingroup
\begin{align}
\mathbb{E}\bigl|a_m^{\mathrm{GAE}}-a_m\bigr|
&\le
K_m\,\mathbb{E}\left(\|\delta(z)\|_2+\|\delta(z_{\setminus m})\|_2+\|\delta(0)\|_2\right)\nonumber\\
&\le
K_m\left(2\|W_{\mathrm{dec}}-W_{\mathrm{dec}}^{\mathrm{ID}}\|_F\,\mathbb{E}\|z\|_2
+3\|b_{\mathrm{dec}}-b_{\mathrm{dec}}^{\mathrm{ID}}\|_2\right),
\label{eq:feature_metric_projection_bound}
\end{align}
\endgroup
The last step uses $\|z_{\setminus m}\|_2\le\|z\|_2$ and the triangle inequality.
Averaging the budget-wise bounds and applying the triangle inequality gives the nAOPC result.
\end{proof}

Proposition~\ref{prop:feature_metric_transfer} controls score changes relative to Fixed under the stated assumptions.
It gives no direction of improvement for nComp or nAOPC, which we evaluate empirically.
Preserving feature activations does not establish the feature-stability condition.

\section{Empirical Evidence for Section~\ref{sec:hidden_geometry_faithfulness}}
\label{sec:empirical_evidence}
This appendix provides empirical support for the theoretical results in Section~\ref{sec:hidden_geometry_faithfulness}. We verify three claims: (i) the explainer subspace aligns with the ID-active subspace, (ii) the explainer-dependent component in the decomposition~\eqref{eq:decomposition} accounts for a meaningful fraction of the total OOD error, and (iii) second-moment shift enlarges the faithfulness gap as predicted by Proposition~\ref{prop:second-moment_shift_enlarges_the_faithfulness_gap}.

\subsection{Controlled Toy Setting}
\label{subsec:toy_setting}
The toy experiments use a controlled setting that allows us to vary OOD severity continuously while keeping the model and explainer fixed.

\paragraph{Target model.}
The target model is a 2-layer ReLU MLP with input dimension $d_{\mathrm{in}}{=}128$, hidden dimension $d{=}256$, and output dimension $p \in \{4, 8, 16\}$:
\[
h(x) = \mathrm{ReLU}(W_1 x + b_1) \in \mathbb{R}^{d}, \qquad o(x) = W_2\,h(x) + b_2 \in \mathbb{R}^p.
\]
The hidden activations $h(x) \in \mathbb{R}^d$ correspond to the hidden representations analyzed in the main text. The explainer operates on these $d$-dimensional activations.

\paragraph{Explainer.}
We train both a transcoder and an SAE on ID hidden activations using the standard reconstruction-plus-sparsity objective (ERM), with dictionary sizes $k \in \{d/2,\, 1d,\, 2d,\, 4d,\, 8d,\, 32d\}$.
The subspace rank is set to $r = p$, matching the rank of the output weight matrix $W_2 \in \mathbb{R}^{p \times d}$: since $o(x) = W_2\,h(x) + b_2$, the model's output depends on $h(x)$ only through its projection onto $\mathrm{span}(W_2^\top)$, which has dimension $p$.
We use $r=p$ as a truncation choice for the hidden second moment.
Its leading eigenspace need not coincide with the output-relevant subspace $\mathrm{span}(W_2^\top)$.

\paragraph{OOD generation.}
ID inputs are drawn from $x \sim \mathcal{N}(0, I_{d_{\mathrm{in}}})$.
OOD inputs are generated as $x = A_s\,z$ where $z \sim \mathcal{N}(0, I_{d_{\mathrm{in}}})$ and $A_s$ is a severity-dependent transformation matrix.
The resulting input covariance is $\Sigma(s)=A_sA_s^\top$, which is symmetric and positive semidefinite.
The covariance is globally normalized so that $\mathrm{tr}(\Sigma(s)) = d_{\mathrm{in}}$ for all $s$, ensuring that reconstruction error differences are not driven by trivial scale changes.
At $s{=}0$, $\Sigma(0) = I$ (ID).
As $s$ increases toward $1$, the input covariance undergoes progressive rotation and anisotropic rescaling, which propagates through the ReLU layer to induce second-moment shift in hidden space.
We use $N{=}20{,}000$ samples for all geometric computations.

\paragraph{Validation on real models.}
We also validate the results on pretrained language models (GPT-2 Small, Pythia-1.4B) under temporal, domain, and adversarial distribution shifts, using the experimental setup described in Section~\ref{sec:experiments}. Unlike the toy setting, OOD severity is not varied continuously; instead, we compare ID and pure OOD for each shift type.

\subsection{Explainer Subspace Alignment with the ID Dominant Subspace}
\label{subsec:explainer_subspace_alignment}
Section~\ref{subsec:notation_preliminaries} claims that the explainer subspace $\Pi$ aligns closely with $\Pi_{\mathrm{ID}}$ for a well-trained ID explainer. We verify this by measuring the subspace overlap
\[
\mathrm{overlap}(U_{\mathrm{dec}},\, U_{\mathrm{ID}})
:= \frac{1}{r}\,\|U_{\mathrm{dec}}^\top U_{\mathrm{ID}}\|_F^2
\;\in\;[0,1],
\]
where $U_{\mathrm{ID}} \in \mathbb{R}^{d \times r}$ contains the top-$r$ eigenvectors of $M_{\mathrm{ID}}$. A value of $1$ indicates perfect alignment; $0$ indicates orthogonality.

\subsubsection{Toy Setting}
\label{subsubsec:alignment_toy}
We sweep dictionary sizes $k \in \{d/2,\, 1d,\, 2d,\, 4d,\, 8d,\, 32d\}$ for both transcoders and SAEs to test whether the claim holds across different levels of overcompleteness. Figure~\ref{fig:subspace_overlap} reports the overlap as a function of OOD severity for each dictionary size. In both panels, solid lines show the explainer--ID overlap and dashed lines show the explainer--OOD overlap.

\begin{figure}[h]
\centering
\begin{subfigure}[t]{0.48\linewidth}
    \centering
    \includegraphics[width=\linewidth]{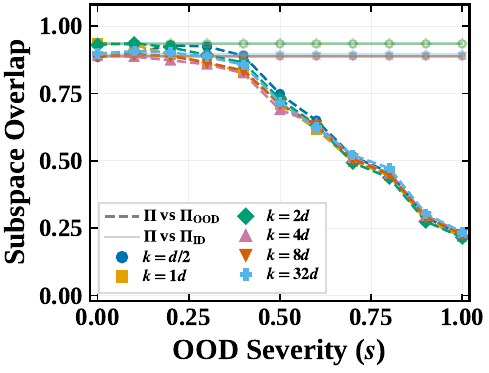}
\end{subfigure}
\hfill
\begin{subfigure}[t]{0.48\linewidth}
    \centering
    \includegraphics[width=\linewidth]{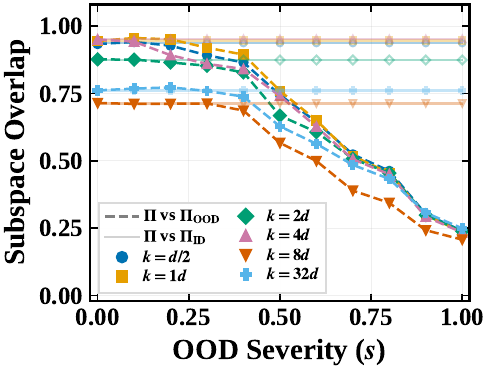}
\end{subfigure}
\caption{\textbf{Explainer subspace overlap between the explainer and the ID-active subspace (solid) vs.\ the OOD-active subspace (dashed), as a function of OOD severity $s$, for dictionary sizes $k \in \{d/2, 1d, 2d, 4d, 8d, 32d\}$.} \textbf{Left:} Transcoder. \textbf{Right:} SAE. For $k \ge 4d$, both explainer types maintain high ID overlap ($>0.89$) regardless of severity, while OOD overlap degrades monotonically.}
\label{fig:subspace_overlap}
\end{figure}

\paragraph{Transcoder (left).} The explainer--ID overlap (solid) remains above $0.89$ for all dictionary sizes and all severity levels. The overlap is largely insensitive to $k$: even an undercomplete dictionary ($k{=}d/2$) captures the ID-active subspace well. The explainer--OOD overlap (dashed) degrades monotonically with severity, dropping below $0.3$ at $s{=}1.0$ regardless of $k$.

\paragraph{SAE (right).} The explainer--ID overlap depends more on $k$. For $k \ge 4d$, the overlap exceeds $0.93$, comparable to the transcoder. For smaller $k$ ($k{=}d/2$ or $k{=}1d$), the overlap drops to $0.70$--$0.77$, indicating that undercomplete SAEs do not fully capture the ID-active subspace. As with the transcoder, the explainer--OOD overlap degrades with severity for all $k$.

\paragraph{Interpretation.} For sufficiently overcomplete dictionaries ($k \ge 4d$, the standard setting in practice), both transcoders and SAEs align closely with the ID-active subspace regardless of OOD severity, confirming the claim in Section~\ref{subsec:notation_preliminaries}. The divergence between the ID overlap (flat) and the OOD overlap (decreasing) is precisely the faithfulness gap: the explainer remains anchored to the ID geometry while the model's active subspace rotates away under OOD shift.

\subsubsection{Real-Data Setting}
\label{subsubsec:alignment_realdata}
We evaluate ID-trained explainers on GPT-2 Small and Pythia-1.4B under temporal, domain, and adversarial shifts.
The rank is $r{=}64$ for both models.

\paragraph{Faithfulness gap on $M_{\mathrm{ID}}$ vs.\ on $M_{\mathrm{OOD}}$.}
Definition~\ref{def:faithfulness_gap} measures the gap of $\Pi_{\mathrm{dec}}$ against $\Pi_{\mathrm{OOD}}$. To assess the residual ID-side misalignment, we additionally report the analogous quantity against $\Pi_{\mathrm{ID}}$,
\[
\Delta_{\mathrm{ID}}(\Pi_{\mathrm{dec}}) := \|\Pi_{\mathrm{ID}} - \Pi_{\mathrm{dec}}\|_F,
\]
which directly quantifies how well the ID-trained explainer captures $\Pi_{\mathrm{ID}}$.
Frobenius gap and overlap encode the same information through $\|\Pi_A - \Pi_B\|_F = \sqrt{2r\,(1-\mathrm{overlap}(U_A, U_B))}$, and Table~\ref{tab:frobenius_gap_real} reports the corresponding gaps.
The maximum value at $r{=}64$ is $\sqrt{2r} = \sqrt{128} \approx 11.31$.

\begin{table}[h]
\centering
\small
\caption{\textbf{Frobenius faithfulness gap of the ID-trained explainer at $r{=}64$.}
Columns 4--6 report $\Delta_{\mathrm{ID}}(\Pi_{\mathrm{dec}}) = \|\Pi_{\mathrm{ID}}-\Pi_{\mathrm{dec}}\|_F$, $\Delta(\Pi_{\mathrm{dec}}) = \|\Pi_{\mathrm{OOD}}-\Pi_{\mathrm{dec}}\|_F$, and $\|\Pi_{\mathrm{ID}}-\Pi_{\mathrm{OOD}}\|_F$.
The last column is the ratio $\Delta(\Pi_{\mathrm{dec}})/\Delta_{\mathrm{ID}}(\Pi_{\mathrm{dec}})$.
Maximum possible value at $r{=}64$ is $\sqrt{128}\approx 11.31$.
All values are derived from the measured subspace overlaps.}
\label{tab:frobenius_gap_real}
\begin{tabular}{llccccc}
\toprule
Explainer & Model & Shift & $\Delta_{\mathrm{ID}}(\Pi_{\mathrm{dec}})$ & $\Delta(\Pi_{\mathrm{dec}})$ & $\|\Pi_{\mathrm{ID}}-\Pi_{\mathrm{OOD}}\|_F$ & Ratio \\
\midrule
\multirow{6}{*}{Transcoder}
 & GPT-2 Small  & Temporal      & 5.89 & 6.17  & 5.65  & 1.05$\times$ \\
 & GPT-2 Small  & Domain        & 5.89 & 9.79  & 9.83  & 1.66$\times$ \\
 & GPT-2 Small  & Adversarial   & 5.89 & 9.97  & 9.83  & 1.69$\times$ \\
 \cmidrule(lr){2-7}
 & Pythia-1.4B  & Temporal      & 7.73 & 8.72  & 7.68  & 1.13$\times$ \\
 & Pythia-1.4B  & Domain        & 7.73 & 10.75 & 10.78 & 1.39$\times$ \\
 & Pythia-1.4B  & Adversarial   & 7.73 & 10.46 & 10.49 & 1.35$\times$ \\
\midrule
\multirow{6}{*}{SAE}
 & GPT-2 Small  & Temporal      & 9.76 & 9.75  & 4.83  & 1.00$\times$ \\
 & GPT-2 Small  & Domain        & 9.76 & 10.98 & 10.33 & 1.12$\times$ \\
 & GPT-2 Small  & Adversarial   & 9.76 & 10.95 & 10.30 & 1.12$\times$ \\
 \cmidrule(lr){2-7}
 & Pythia-1.4B  & Temporal      & 9.95 & 10.20 & 7.65  & 1.03$\times$ \\
 & Pythia-1.4B  & Domain        & 9.95 & 11.10 & 10.88 & 1.12$\times$ \\
 & Pythia-1.4B  & Adversarial   & 9.95 & 11.03 & 10.71 & 1.11$\times$ \\
\bottomrule
\end{tabular}
\end{table}

For transcoders, $\Delta_{\mathrm{ID}}(\Pi_{\mathrm{dec}})$ is consistently smaller than $\Delta(\Pi_{\mathrm{dec}})$, with the OOD-side gap exceeding the ID-side gap by $1.35$--$1.69\times$ under domain and adversarial shifts, where the second-moment shift is largest.
Moreover, $\|\Pi_{\mathrm{ID}}-\Pi_{\mathrm{OOD}}\|_F$ tracks $\Delta(\Pi_{\mathrm{dec}})$ to within $3\%$ under these shifts.
The residual ID-side misalignment still enters the triangle-inequality bound in Section~\ref{subsec:second-moment_shift_enlarges_the_faithfulness_gap_and_degrades_faithfulness}.
SAE explainers exhibit a larger $\Delta_{\mathrm{ID}}(\Pi_{\mathrm{dec}})$ (close to the upper bound at $r{=}64$), reflecting the lower per-rank overlap of residual-stream dictionaries; the ordering $\Delta(\Pi_{\mathrm{dec}}) \ge \Delta_{\mathrm{ID}}(\Pi_{\mathrm{dec}})$ still holds under domain and adversarial shifts but with a smaller margin ($1.11$--$1.12\times$).
Temporal shift produces nearly equal ID- and OOD-side gaps for both explainer types, consistent with its milder second-moment perturbation.


\subsection{Relative Magnitude of the Explainer-Dependent Term}
\label{subsec:explainer_dependent_magnitude}

The decomposition of Equation~\eqref{eq:decomposition} separates OOD faithfulness as
\[
\mathcal{L}_{\mathrm{OOD}}(\Pi)
=
\underbrace{\mathcal{L}_{\mathrm{OOD}}(\Pi_{\mathrm{OOD}})}_{\text{irreducible}}
+
\underbrace{\operatorname{tr}\!\big((\Pi_{\mathrm{OOD}}-\Pi)\,M_{\mathrm{OOD}}\big)}_{\text{explainer-dependent}}.
\]
Adaptation can only reduce the explainer-dependent component. If this component is negligible relative to the irreducible component, no adaptation strategy can meaningfully improve OOD faithfulness. We measure the explainer-dependent ratio
\[
\eta := \frac{\operatorname{tr}[(\Pi_{\mathrm{OOD}} - \Pi)\,M_{\mathrm{OOD}}]}{\mathcal{L}_{\mathrm{OOD}}(\Pi)}
\]
to assess whether the explainer-dependent component is an actionable target.

\subsubsection{Toy Setting}
\label{subsubsec:explainer_dep_toy}

In the toy setting, $\eta$ rises from below $0.05$ for most $k$ at $s=0$ to approximately $0.31$ at $s=1$, with little variation across dictionary sizes.
The moderate value reflects the low rank ratio $r/d=8/256$, which leaves a large irreducible component in this setting.

\subsubsection{Real-Data Setting}
\label{subsubsec:explainer_dep_realdata}
Figure~\ref{fig:thm1_decomposition_real} reports $\eta$ at pure ID and pure OOD for GPT-2 Small and Pythia-1.4B under temporal, domain, and adversarial shifts.

\paragraph{Transcoder (left).} At pure OOD, domain and adversarial shifts reach $\eta > 0.99$ across both models: the explainer-dependent component dominates the total error almost entirely. Temporal shift yields $\eta \approx 0.66$--$0.99$ depending on the model, reflecting its milder geometric distortion. These values are substantially higher than in the toy setting ($\eta \approx 0.31$), because the toy setting uses $r/d = 8/256 = 3.1\%$ so the irreducible term dominates by construction.

\paragraph{SAE (right).}
SAE explainers exhibit the same trend. Under domain and adversarial shifts, $\eta > 0.99$ for both models, confirming that the explainer-dependent component dominates. Temporal shift yields $\eta \approx 0.68$--$0.99$ depending on the model, consistent with its milder geometric distortion. At pure ID, $\eta \approx 0.51$ (GPT-2) and $\eta \approx 0.99$ (Pythia-1.4B), reflecting the larger model's higher effective dimensionality relative to $r{=}64$.

\begin{figure}[h]
\centering
\begin{subfigure}[t]{0.48\linewidth}
    \centering
    \includegraphics[width=\linewidth]{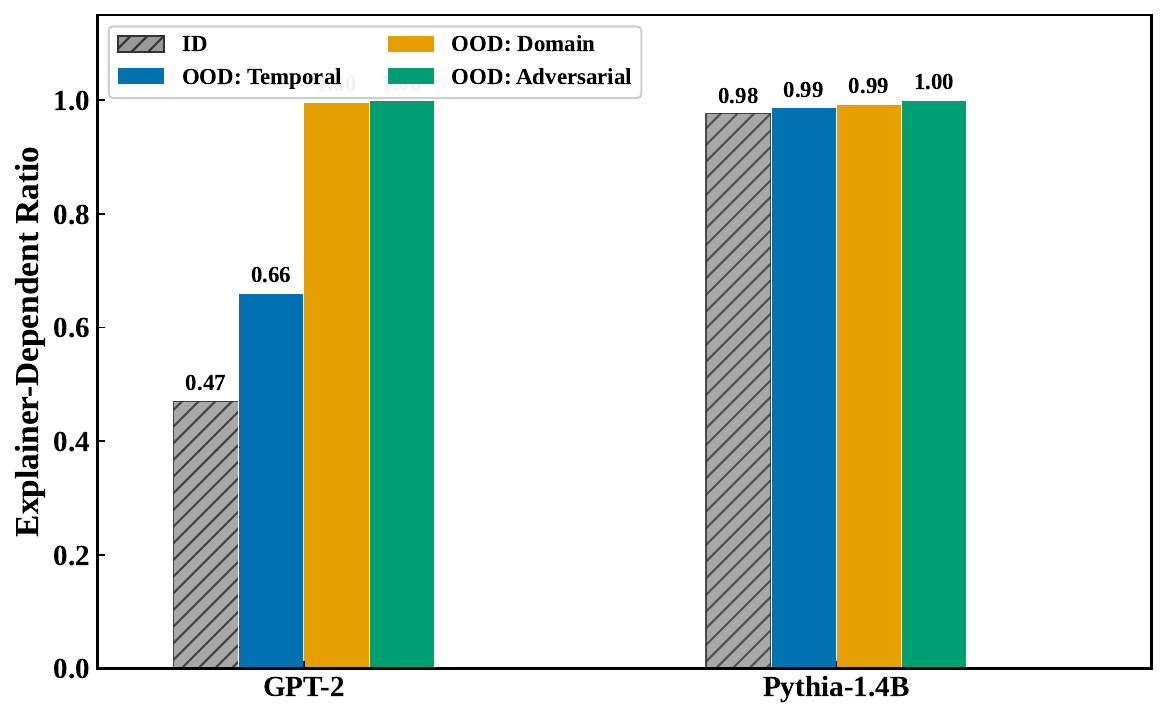}
    \caption{Transcoder}
\end{subfigure}
\hfill
\begin{subfigure}[t]{0.48\linewidth}
    \centering
    \includegraphics[width=\linewidth]{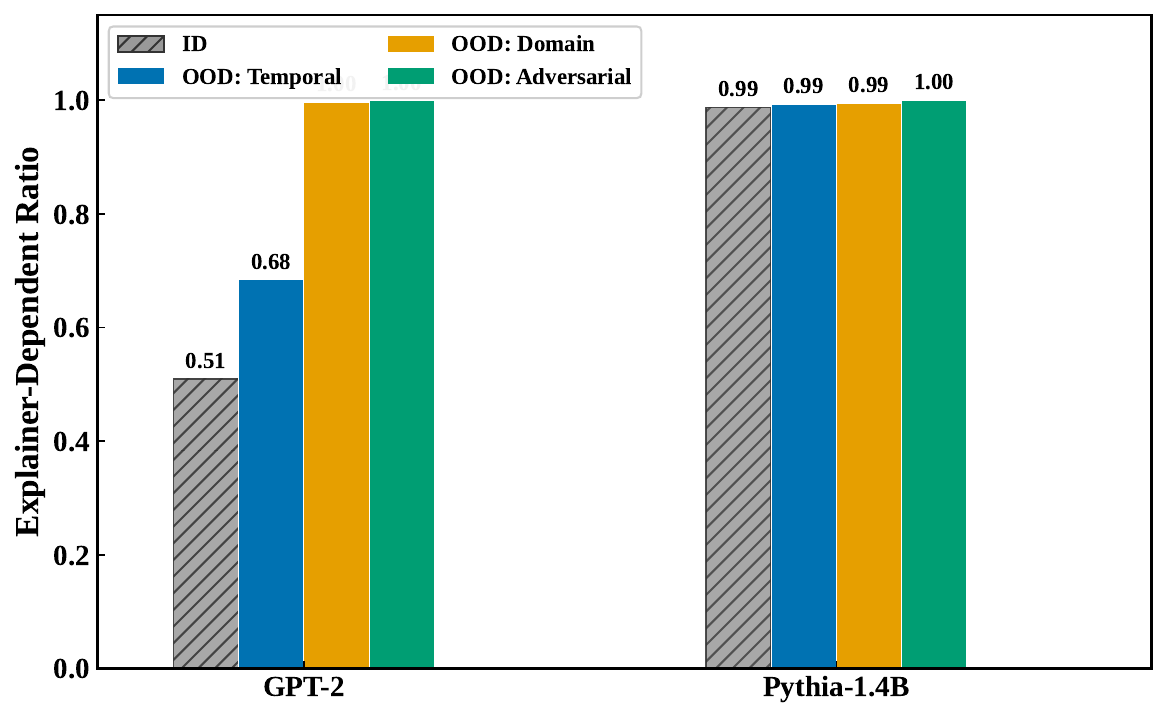}
    \caption{SAE}
\end{subfigure}
\caption{\textbf{Explainer-dependent ratio $\eta$ at pure ID (hatched) and pure OOD (colored) ($r{=}64$).} Under domain and adversarial shifts, $\eta > 0.99$ for both explainer types.}
\label{fig:thm1_decomposition_real}
\end{figure}


\subsection{Empirical Verification of Proposition~\ref{prop:second-moment_shift_enlarges_the_faithfulness_gap}}
\label{subsec:empirical_verification_propositions}
Proposition~\ref{prop:second-moment_shift_enlarges_the_faithfulness_gap} predicts that second-moment shift controls the faithfulness gap via
\[
\Delta(\Pi_{\mathrm{ID}}) \le \frac{2}{\gamma_{\mathrm{ID}}}\,\|M_{\mathrm{OOD}} - M_{\mathrm{ID}}\|_F.
\]
Since this bound depends only on $M_{\mathrm{ID}}$ and $M_{\mathrm{OOD}}$, it is independent of the explainer architecture and dictionary size. We verify it empirically on both the toy setting and real models.

\subsubsection{Toy Setting}
\label{subsubsec:prop_verification_toy}

In the toy setting, normalized second-moment shift and $\Delta(\Pi_{\mathrm{ID}})$ are strongly correlated (Pearson $r=0.993$, Spearman $\rho=1.000$).

\subsubsection{Real-Data Setting}
\label{subsubsec:prop_verification_realdata}

\begin{figure}[!htbp]
\centering
\begin{subfigure}[t]{0.48\linewidth}
    \centering
    \includegraphics[width=\linewidth]{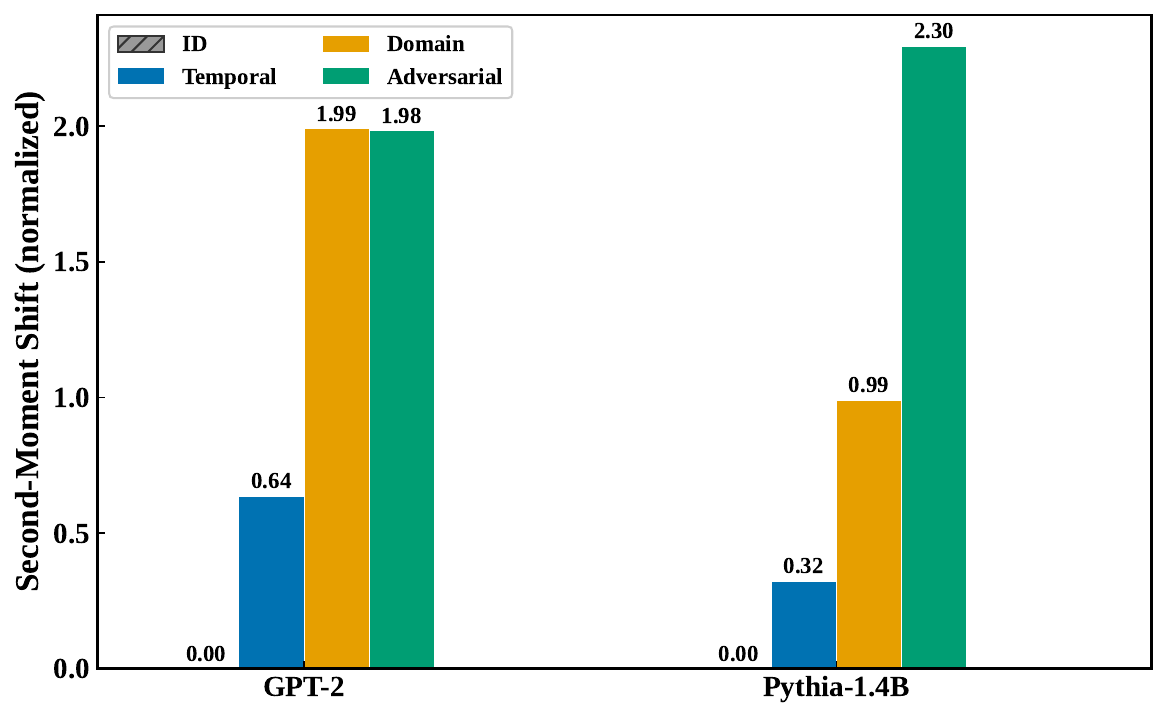}
    \caption{Transcoder: Second-moment shift}
\end{subfigure}
\hfill
\begin{subfigure}[t]{0.48\linewidth}
    \centering
    \includegraphics[width=\linewidth]{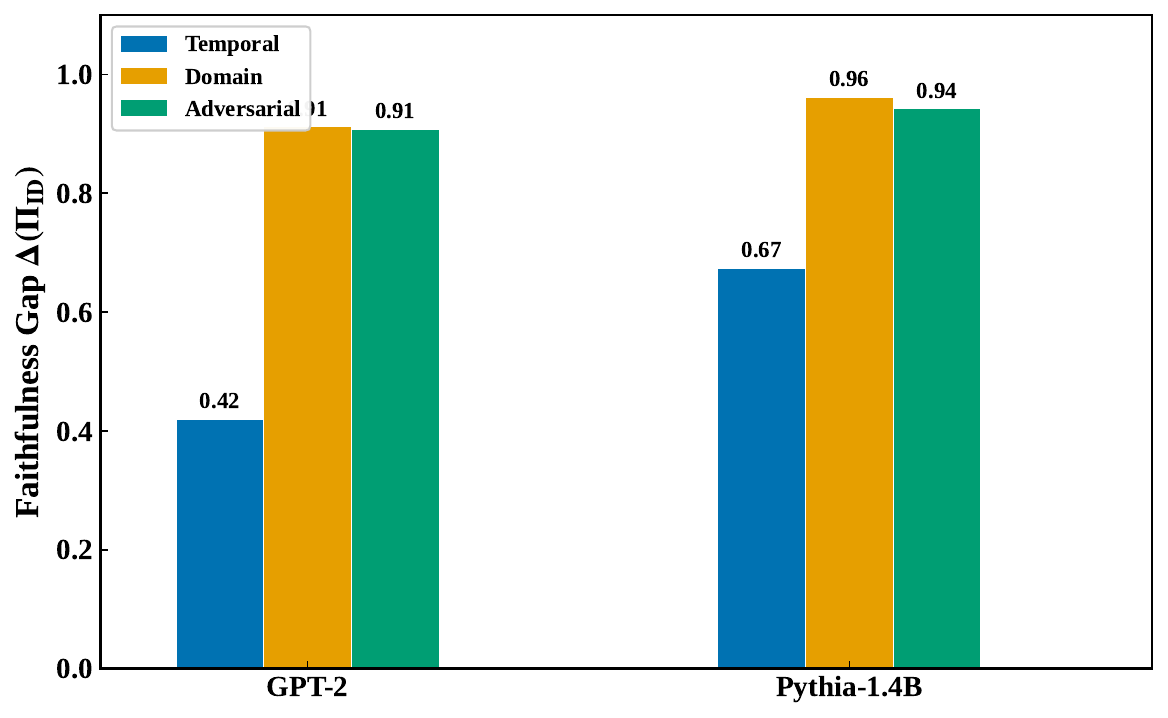}
    \caption{Transcoder: Faithfulness gap $\Delta(\Pi_{\mathrm{ID}})$}
\end{subfigure}
\\
\begin{subfigure}[t]{0.48\linewidth}
    \centering
    \includegraphics[width=\linewidth]{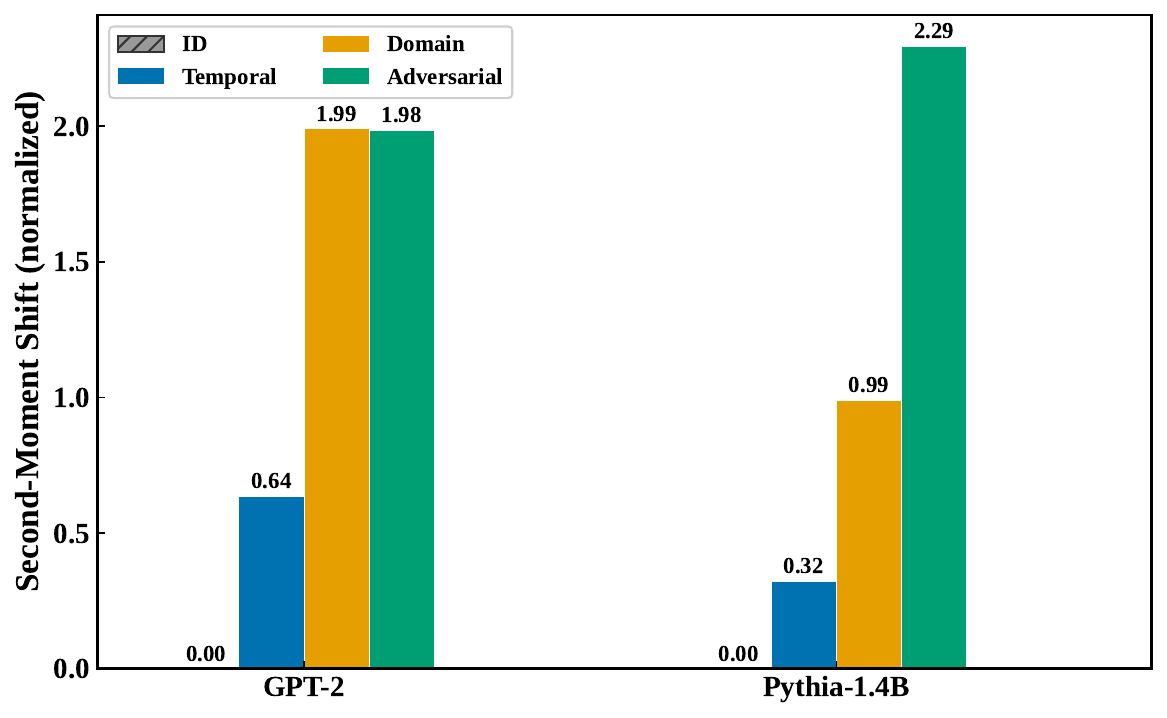}
    \caption{SAE: Second-moment shift}
\end{subfigure}
\hfill
\begin{subfigure}[t]{0.48\linewidth}
    \centering
    \includegraphics[width=\linewidth]{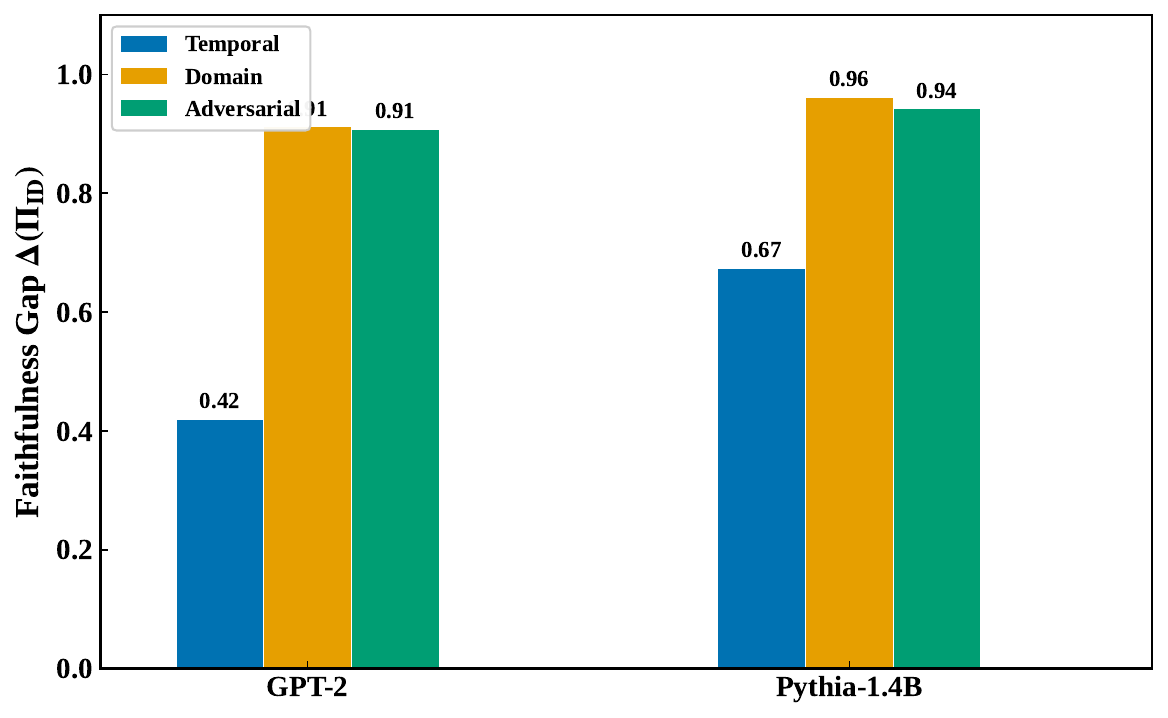}
    \caption{SAE: Faithfulness gap $\Delta(\Pi_{\mathrm{ID}})$}
\end{subfigure}
\caption{\textbf{Proposition~\ref{prop:second-moment_shift_enlarges_the_faithfulness_gap} verification ($r{=}64$) at pure OOD.} Top row: Transcoder. Bottom row: SAE. Within each model, larger shifts correspond to larger gaps for both explainer types.}
\label{fig:prop2_verification_real}
\end{figure}

Figure~\ref{fig:prop2_verification_real} shows the second-moment shift and faithfulness gap for GPT-2 Small and Pythia-1.4B ($r{=}64$) at pure OOD for both transcoders (top row) and SAEs (bottom row). Within each model, the ordering is consistent: temporal shift produces the smallest second-moment shift and the smallest faithfulness gap, while domain and adversarial shifts produce larger shifts and correspondingly larger gaps.

\paragraph{SAE.} The same ordering holds for SAE explainers: temporal shift produces the smallest second-moment shift and faithfulness gap, while domain and adversarial shifts produce larger values. The magnitudes are comparable to transcoders, confirming that the proposition is independent of the explainer architecture.

\paragraph{Summary.}
The proposition holds across all models and both explainer types under diverse real-world distribution shifts, confirming that the theoretical predictions generalize beyond the controlled toy setting.

\section{Empirical Illustration of Theorem~\ref{thm:improvement_over_id}}
\label{subsec:geometric_appendix}
Under exact subspace estimation, Theorem~\ref{thm:improvement_over_id} lower-bounds the Step~1 projection-loss improvement by $\tfrac{1}{2}\gamma_{\mathrm{OOD}}\,\Delta(\Pi_{\mathrm{ID}})^2$.
We illustrate this relationship in the controlled toy setting (Section~\ref{subsec:controlled_validation}) by fixing $r=p=8$ and sweeping OOD severity from $s{=}0$ (ID) to $s{=}1$ (maximum shift).
At each severity, we use the Step~1 projector $\widehat{\Pi}_{\mathrm{OOD}}$ and measure the projection-loss improvement $I(s) = \mathcal{L}_{\mathrm{OOD},s}(\Pi_{\mathrm{ID}}) - \mathcal{L}_{\mathrm{OOD},s}(\widehat{\Pi}_{\mathrm{OOD}})$ against the squared faithfulness gap $\Delta(\Pi_{\mathrm{ID}})^2$.

\begin{figure}[!htbp]
\centering
\includegraphics[width=0.7\linewidth]{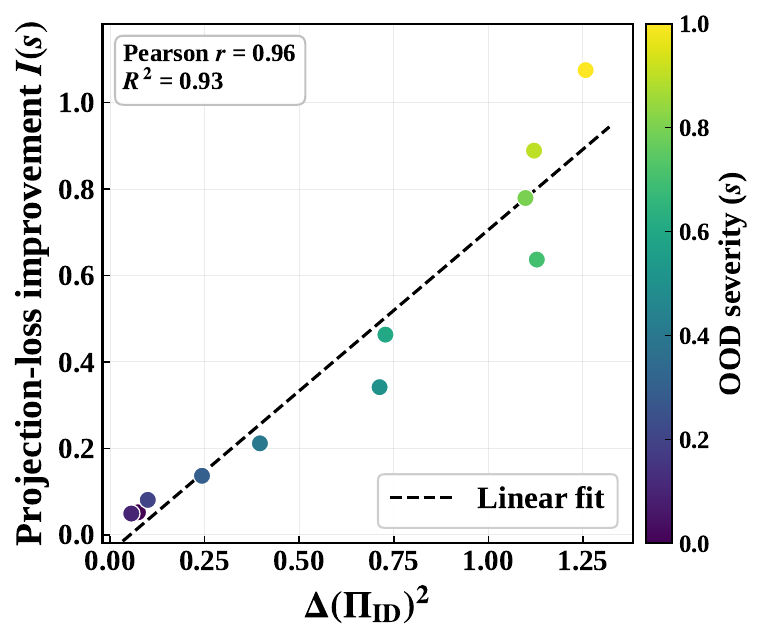}
\caption{\textbf{Empirical illustration of Theorem~\ref{thm:improvement_over_id} on the controlled toy setting.}
Projection-loss improvement $I(s)$ versus the squared faithfulness gap $\Delta(\Pi_{\mathrm{ID}})^2$, swept across OOD severity $s \in [0,1]$.
The dashed line is a linear fit ($R^2 = 0.93$, Pearson $r = 0.96$), illustrating the Step~1 relationship.}
\label{fig:theorem1_verification}
\end{figure}

Figure~\ref{fig:theorem1_verification} shows a strong linear relationship between $I(s)$ and $\Delta(\Pi_{\mathrm{ID}})^2$ (Pearson $r = 0.96$, $R^2 = 0.93$), illustrating the Step~1 relationship.
The empirical improvement exceeds the population lower-bound expression at every severity (0 violations out of 11).
This finite-sample comparison illustrates Step~1 behavior under estimated OOD subspaces.
The reference curve uses the eigengap coefficient $\gamma_{\mathrm{OOD}}/2$, which is conservative relative to the effective improvement rate, as expected for a spectral-gap-based guarantee.

\section{Experimental Details}\label{sec:experimental_details}

\subsection{Model and Explainer Details}
\label{subsec:model_explainer_details}

Table~\ref{tab:model_details} summarizes the model and explainer configurations. All models are frozen pretrained checkpoints; only explainer components are adapted.

\begin{table}[h]
\centering
\caption{\textbf{Model and explainer configurations.}}
\label{tab:model_details}
\begin{tabular}{lccccl}
\toprule
Model & $d$ & Layer & $k$ ($32d$) & ID Corpus \\
\midrule
GPT-2 Small & 768 & 8 & 24{,}576 & OpenWebText~\citep{Gokaslan2019OpenWeb} \\
Pythia-1.4B & 2{,}048 & 15 & 65{,}536 & The Pile~\citep{gao2020pile} \\
\bottomrule
\end{tabular}
\end{table}

For each model, we train two explainer types: Top-K SAEs~\citep{gao2024scaling}, which reconstruct residual-stream activations $h$ using Top-K sparsity, and transcoders~\citep{dunefsky2024transcoders}, which reconstruct MLP outputs $h$ from sparse feature activations $z$ extracted from MLP inputs.
All explainers use dictionary size $k{=}32d$~\citep{bricken2023monosemanticity} and are trained on in-distribution activations with the standard reconstruction-plus-sparsity objective.

\subsection{Baseline Details}
\label{subsec:baseline_details}

\begin{itemize}[leftmargin=1.5em]
\item \textbf{TERM.} Tilted empirical risk minimization uses only ID data~\citep{li2020tilted,muhamed2025decoding}.

\item \textbf{Retrain.} The explainer retrained from scratch on OOD activations with the same architecture and hyperparameters.
This baseline provides a reference point but is not an oracle upper bound: retraining on OOD data can distort pretrained feature structure~\citep{kumar2022fine}.

\item \textbf{SAEBoost.} A residual boosting approach~\citep{koriagin2025teach}: a secondary explainer is trained on the OOD reconstruction residuals of the ID-trained base explainer, and the two outputs are summed at inference ($\hat{h} = \hat{h}_{\mathrm{base}} + \hat{h}_{\mathrm{resid}}$).
This adds OOD-specific capacity while retaining the base dictionary, but requires OOD training data.

\item \textbf{FaithfulSAE.} Full retraining uses only the model's own unconditional generations~\citep{cho2025faithfulsae}.
\end{itemize}

\subsection{Comparison with Geometric Domain Adaptation}
\label{subsec:uda_comparison}

The geometric baselines use the same unlabeled OOD activation cache and evaluation protocol as GAE.
Table~\ref{tab:uda_results} reports the HaluEval transcoder results in the main text.

\paragraph{Baseline definitions.}
PCA estimates the rank-$r$ OOD principal subspace and uses its signed coefficients for evaluation.
Its dense coordinates differ from sparse dictionary features in sparsity and feature count, so this comparison is approximate.
For SA~\citep{fernando2013unsupervised} and CORAL~\citep{sun2016return}, we keep the sparse encoder frozen and transform only the decoder.
These baselines are decoder adaptations of the original domain adaptation methods.

\paragraph{Additional FineWeb evaluation.}
Table~\ref{tab:uda_fineweb} compares the same methods on the Pythia-1.4B SAE under FineWeb.
GAE achieves higher nAOPC and nComp and lower $|\Delta\mathrm{CE}|$ than PCA, SA, and CORAL in this setting.

\begin{table}[ht]
\centering
\caption{\textbf{Geometric adaptation on FineWeb (Pythia-1.4B, SAE).}
GAE improves both causal faithfulness and reconstruction quality over PCA, SA, and CORAL.
\textbf{Bold}: best. \underline{Underline}: second best.}
\label{tab:uda_fineweb}
\begin{tabular}{lccc}
\toprule
\textbf{Method} & \textbf{nAOPC}$\uparrow$ & \textbf{nComp}$\uparrow$ & $|\boldsymbol{\Delta}\textbf{CE}|\!\downarrow$ \\
\midrule
Fixed & \underline{0.962} & \underline{1.536} & \underline{0.0216} \\
PCA & 0.959 & 0.978 & 0.1324 \\
SA & 0.841 & 1.090 & 0.1301 \\
CORAL & 0.953 & 1.247 & 0.1722 \\
\rowcolor{gaerowhighlight} \textbf{GAE} & \textbf{0.985} & \textbf{1.677} & \textbf{0.0207} \\
\bottomrule
\end{tabular}
\end{table}

\subsection{GAE Implementation Details}
\label{subsec:gae_implementation_details}

\paragraph{Step~2 dictionary preservation.}
The dictionary preservation penalty in the closed-form decoder refit (Section~\ref{subsec:gae_methods}, Step~2) regularizes the decoder toward the Step~1 output $\widetilde{W}_{\mathrm{dec}}(T^\star)$ with weight $\lambda_{\mathrm{pres}}$, following Equation~\eqref{eq:gae_step2_dec}.

\paragraph{Decoder interpolation.}
The Step~2 closed-form solution $W_{\mathrm{dec}}^{\mathrm{GAE}}$ optimizes sample-level reconstruction under geometry penalties.
With limited OOD samples, this solution can overfit to the estimation noise in $\{z_i, h_i\}$.
We use the decoder interpolation in Equation~\eqref{eq:decoder_interpolation} to limit this effect.
We treat $\alpha$ as a hyperparameter selected per OOD setting.

\paragraph{Hyperparameter selection.}
GAE requires no gradient computation or iterative optimization. The hyperparameters ($r$, $\lambda_{\mathrm{geom}}$, $\lambda_{\mathrm{pres}}$, $\alpha$) are selected per OOD setting using a small held-out portion of unlabeled OOD activations, monitoring reconstruction quality ($|\Delta\mathrm{CE}|$) and causal faithfulness (nComp). We sample the OOD validation/adaptation split with a fixed seed from the same OOD source. It contains $N_{\mathrm{cov}}{=}2{,}000$ prompts for estimating $\widehat{M}_{\mathrm{OOD}}$ and $N_{\mathrm{fit}}$ activations for Step~2. No OOD labels are required. The final test split contains $N_{\mathrm{eval}}{=}1{,}000$ \textit{disjoint} prompts. Once selected, the same hyperparameters are used for all evaluation prompts.

\paragraph{Hyperparameter summary.}
Tables~\ref{tab:gae_hparams_tc} and~\ref{tab:gae_hparams_sae} list the GAE hyperparameters for transcoders and SAEs, respectively. The search grid uses $r \in \{4,8,32,64,256,512,1024\}$ for transcoders and explainer-specific SAE rank grids, $\lambda_{\mathrm{geom}}\in\{0,0.1,1,15,20\}$, $\lambda_{\mathrm{pres}}\in\{0,0.04,0.2,1,2\}$, $\alpha\in\{0,1\}$, and $N_{\mathrm{fit}}\in\{0,2048\}$. All settings use the second-moment matrix (not centered covariance) for OOD subspace estimation, as prescribed in Algorithm~\ref{alg:gae}.

\begin{table}[h]
\centering
\caption{\textbf{GAE hyperparameters for transcoder experiments.}
$N_{\mathrm{fit}}$ counts hidden-state activation vectors used for Step~2.}
\label{tab:gae_hparams_tc}
\begin{tabular}{llccccc}
\toprule
Model & OOD Setting & $r$ & $\lambda_{\mathrm{geom}}$ & $\lambda_{\mathrm{pres}}$ & $\alpha$ & $N_{\mathrm{fit}}$ \\
\midrule
GPT-2 & HaluEval (Adversarial) & 32 & 0.1 & 0.2 & 0 & 2{,}048 \\
GPT-2 & Edgar (Domain) & 4 & 0.1 & 0.2 & 0 & 2{,}048 \\
GPT-2 & FineWeb (Temporal) & 8 & 0.1 & 0.04 & 0 & 2{,}048 \\
\midrule
Pythia-1.4B & HaluEval (Adversarial) & 64 & 15 & 2 & 0 & 2{,}048 \\
Pythia-1.4B & Edgar (Domain) & 64 & 0.1 & 0.2 & 0 & 2{,}048 \\
Pythia-1.4B & FineWeb (Temporal) & 64 & 20 & 1 & 0 & 2{,}048 \\
\bottomrule
\end{tabular}
\end{table}

\begin{table}[h]
\centering
\caption{\textbf{GAE hyperparameters for SAE experiments.}
$N_{\mathrm{fit}}$ counts hidden-state activation vectors used for Step~2.
In Step~1-Only (Procrustes rotation), decoder refitting is skipped, $N_{\mathrm{fit}}=0$, and $\alpha=1$.
The unused Step~2 regularization weights are listed as zero.}
\label{tab:gae_hparams_sae}
\begin{tabular}{llccccc}
\toprule
Model & OOD Setting & $r$ & $\lambda_{\mathrm{geom}}$ & $\lambda_{\mathrm{pres}}$ & $\alpha$ & $N_{\mathrm{fit}}$ \\
\midrule
GPT-2 & HaluEval (Adversarial) & 512 & 0 & 0 & 1 & 0 \\
GPT-2 & Edgar (Domain) & 512 & 0 & 0 & 1 & 0 \\
GPT-2 & FineWeb (Temporal) & 512 & 0 & 0 & 1 & 0 \\
\midrule
Pythia-1.4B & HaluEval (Adversarial) & 4 & 0.1 & 0.2 & 0 & 2{,}048 \\
Pythia-1.4B & Edgar (Domain) & 1{,}024 & 0 & 0 & 1 & 0 \\
Pythia-1.4B & FineWeb (Temporal) & 4 & 0 & 0.2 & 0 & 2{,}048 \\
\bottomrule
\end{tabular}
\end{table}

\subsection{Evaluation Details}
\label{subsec:evaluation_details}

\paragraph{Normalized comprehensiveness (nComp).}
We measure causal faithfulness via logit-level feature ablation~\citep{bricken2023monosemanticity,cunningham2023sparse}. Given a prompt, let $\ell_0$ denote the target-token logit under the explainer's full reconstruction, and $\ell_\varnothing$ the logit when all features are ablated to zero. For a feature budget $m^*$, let $\ell_{\setminus m^*}$ be the logit after removing the top-$m^*$ features. We define
\begin{equation}
\mathrm{nComp} = \frac{\ell_0 - \ell_{\setminus m^*}}{\left|\ell_0 - \ell_\varnothing\right|}, \label{eq:ncomp}
\end{equation}
where $m^* = 32$. Higher values indicate that the top features are causally important for the model's output.

\paragraph{Delta cross-entropy ($\Delta$CE).}
We measure reconstruction quality by the cross-entropy change when original activations are replaced with the explainer's reconstruction~\citep{gao2024scaling}:
\begin{equation}
\Delta\mathrm{CE} = \mathrm{CE}(\hat{h}) - \mathrm{CE}(h), \label{eq:delta_ce}
\end{equation}
where $\mathrm{CE}(h)$ is the loss with original activations and $\mathrm{CE}(\hat{h})$ is the loss with reconstructed activations.
We report $|\Delta\mathrm{CE}|$, with values closer to zero indicating better reconstruction quality.

\paragraph{Normalized AOPC (nAOPC).}
nAOPC~\citep{edin2025normalized} averages the normalized logit drop across multiple feature budgets when top-$m$ features are removed:
\begin{equation}
\mathrm{nAOPC} = \frac{1}{|\mathcal{M}|}\sum_{m \in \mathcal{M}} \frac{\ell_0 - \ell_{\setminus m}}{\left|\ell_0 - \ell_\varnothing\right|}, \label{eq:naopc}
\end{equation}
where $\mathcal{M} = \{1, 2, 4, 8, 16, 32, 64, 128\}$. Higher values indicate that the identified features are causally important across a range of budgets.

\paragraph{Evaluation protocol.}
We evaluate at the last token position using zero-residual ablation (replacing ablated features with zero). The denominator $|\ell_0 - \ell_\varnothing|$ normalizes each example by the logit range between full and empty reconstruction, enabling cross-example comparison regardless of the absolute logit scale. We exclude examples where $|\ell_0 - \ell_\varnothing| < 0.1$ to avoid unstable normalization. Feature budgets are $\mathcal{M} = \{1, 2, 4, 8, 16, 32, 64, 128\}$ with $m^* = 32$. We use $N_{\mathrm{eval}} = 1{,}000$ evaluation prompts per setting and seed $= 2026$ throughout.

\subsection{Compute Resources}
\label{subsec:compute_resources}

All experiments run on a single NVIDIA RTX A6000 GPU (48\,GiB VRAM) with an Intel Xeon Gold 6326 CPU and 252\,GiB of system RAM. No experiment requires multi-GPU or model-parallel execution. GPT-2 runs use peak GPU memory under 8\,GiB. Pythia-1.4B runs with batch size 64 use peak GPU memory under 24\,GiB.

\paragraph{Per-method wall-clock.}
Table~\ref{tab:compute_cost} reports the cost of a single (model, OOD setting) run for each adaptation method.
GAE times are measured with fixed hyperparameters and exclude hyperparameter search.
GAE uses a single forward pass over $\sim$2{,}000 OOD activations and no gradient computation.
Faithfulness evaluation (nAOPC, nComp, $\Delta$CE on 1{,}000 prompts) adds about 1 minute per (model, OOD setting, baseline).

\paragraph{Total compute.}
The full result table (two models, three OOD settings, six adaptation baselines) requires roughly 50 GPU-hours on a single RTX A6000, dominated by the Retrain-style baselines on Pythia-1.4B. The GAE rows themselves contribute under 1 GPU-minute to this total. Pretraining of the ID dictionaries (transcoders and SAEs on OpenWebText / The Pile) is a one-time cost that we treat as external to the adaptation experiments.

\vspace{-1em}
\section{\texorpdfstring{Additional Case Study and Concept-Level Retention}{Additional Case Study and Concept-Level Retention}}
\label{sec:case_study_appendix}
\vspace{-1em}
This appendix gives the protocol for Section~\ref{subsec:case_study} and applies it to a HaluEval prompt predicting a male first name.
We keep the encoder frozen and select the top-3 features by GAE causal effect on the target token.
Fixed and GAE use identical feature IDs and activations for each prompt.

Each feature's direct logit attribution (DLA) is its activation multiplied by the dot product of its decoder column and the candidate token's unembedding vector.
For each feature, we subtract its mean DLA on 10 unrelated noun controls from its mean DLA on 20 class-member tokens.
The total class-specificity score sums these differences over the three selected features.

\begin{figure}[h]
\centering
\includegraphics[width=\linewidth]{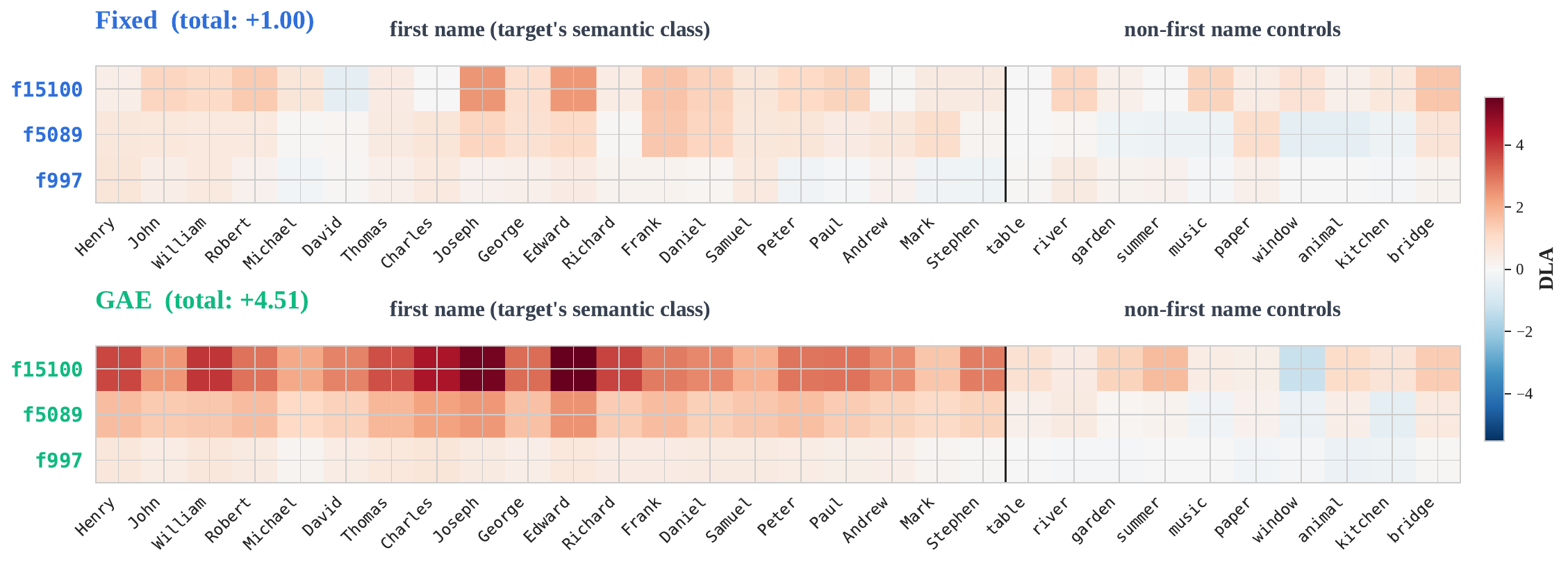}
\caption{\textbf{Per-feature DLA on a prompt predicting `\,Henry' (GPT-2, Transcoder).}
The truncated input is ``Question: What nationality was James''; the next token is a male first name.
Each cell reports the feature's direct logit attribution to a candidate token, with 20 male first names on the left and 10 unrelated noun controls on the right.
Fixed's total class-specificity is $+1.00$ and GAE's is $+4.51$, a $4.5\times$ amplification of the same selected features' pull toward the first-name class.}
\label{fig:case_study_first_name}
\end{figure}

For this prompt, the same three features with unchanged activations contribute more to the first-name class under GAE than under Fixed.

\paragraph{Concept-level retention.}
The preceding case study selects features by their causal effect under GAE.
To assess retention of the original concept relationships, we select three features per concept family using only the Fixed explainer on a disjoint ID split.
We keep these feature IDs fixed across Fixed, Step~1-Only, and GAE, and evaluate them on unseen HaluEval prompts.
Feature selection uses neither GAE scores nor OOD evaluation prompts.

\begin{table}[ht]
\centering
\caption{\textbf{Concept-level retention on unseen HaluEval prompts.}
Fixed, Step~1-Only, and GAE columns report mean class-specific DLA for the same ID-selected features.
Sign retention and correlation compare Fixed with GAE.
The 95\% confidence intervals are paired-bootstrap intervals for the mean GAE-minus-Fixed DLA difference.}
\label{tab:concept_retention}
\begin{tabular}{lrrrrrr}
\toprule
Concept & \multicolumn{3}{c}{Mean DLA $\uparrow$} & \shortstack{Sign\\retention} & Corr. & 95\% CI \\
\cmidrule(lr){2-4}
& Fixed & Step~1-Only & GAE & & & \\
\midrule
Nationality & 1.31 & 1.24 & \textbf{1.39} & 94.4\% & 0.89 & [0.03, 0.13] \\
First name & 1.68 & 1.57 & \textbf{1.72} & 96.1\% & 0.93 & [0.01, 0.08] \\
Profession & 1.18 & 1.10 & \textbf{1.27} & 92.8\% & 0.86 & [0.04, 0.15] \\
\bottomrule
\end{tabular}
\end{table}

Table~\ref{tab:concept_retention} shows that GAE retains the Fixed DLA sign on 92.8--96.1\% of feature--prompt pairs, with Fixed--GAE DLA correlations of 0.86--0.93.
Mean class-specific DLA increases for all three concept families, and all paired-bootstrap confidence intervals exclude zero.
These results support concept-selective retention for the evaluated features and prompts, without establishing universal semantic or causal invariance.

\section{Hyperparameter Sensitivity}
\label{subsec:hp_sensitivity}

We examine GAE's sensitivity to OOD sample size $N_{\mathrm{OOD}}$ and dictionary preservation weight $\lambda_{\mathrm{pres}}$ on HaluEval with GPT-2 transcoders (Figure~\ref{fig:hp_sensitivity}).
We additionally compare rank sensitivity for GPT-2 transcoders and Top-K SAEs (Figure~\ref{fig:rank_sensitivity}).

\begin{figure}[h]

\centering
\begin{subfigure}[t]{0.40\linewidth}
\centering
\includegraphics[width=\linewidth]{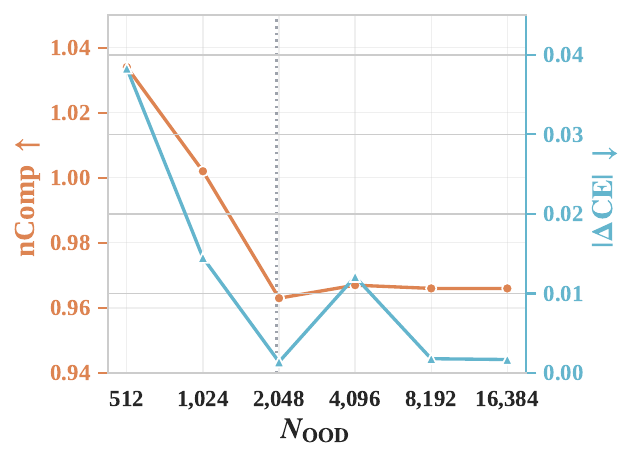}
\caption{$N_{\mathrm{OOD}}$.}
\label{fig:ablation_n_ood}
\end{subfigure}
\qquad
\begin{subfigure}[t]{0.40\linewidth}
\centering
\includegraphics[width=\linewidth]{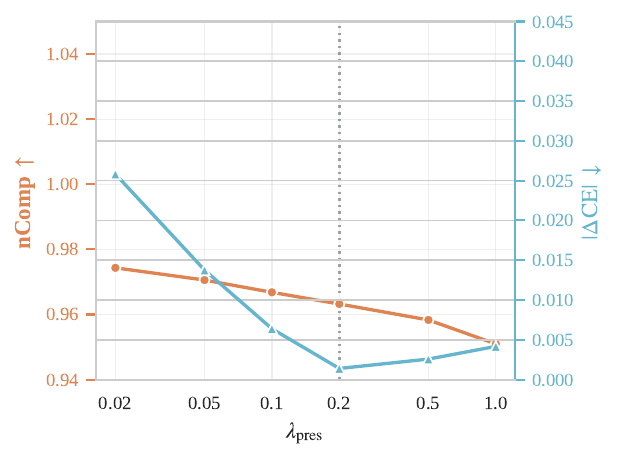}
\caption{$\lambda_{\mathrm{pres}}$.}
\label{fig:ablation_lambda_pres}
\end{subfigure}
\caption{\textbf{Hyperparameter sweeps on HaluEval (GPT-2, Transcoder).}
nComp (orange, left axis) and $|\Delta\mathrm{CE}|$ (cyan, right axis) respond differently to the adaptation settings, motivating joint validation of causal faithfulness and reconstruction quality.}
\label{fig:hp_sensitivity}
\end{figure}

\begin{figure}[ht]
\centering
\begin{subfigure}[t]{0.48\linewidth}
\centering
\includegraphics[width=\linewidth]{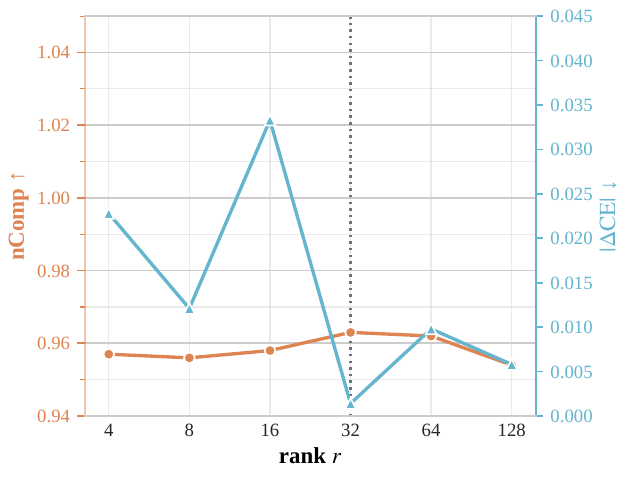}
\caption{Transcoder.}
\label{fig:rank_sensitivity_transcoder}
\end{subfigure}
\hfill
\begin{subfigure}[t]{0.48\linewidth}
\centering
\includegraphics[width=\linewidth]{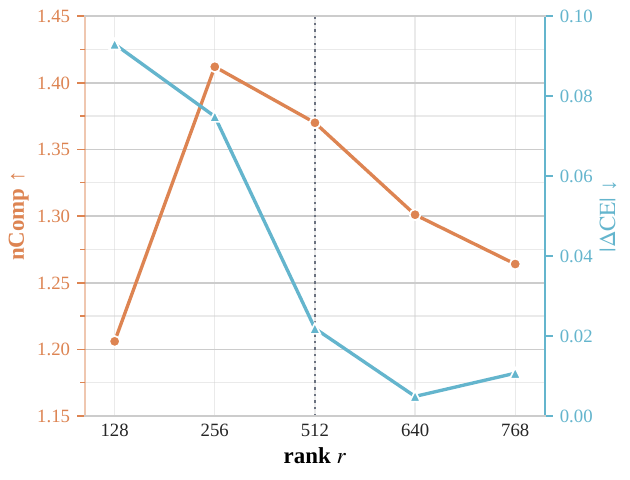}
\caption{Top-K SAE.}
\label{fig:rank_sensitivity_sae}
\end{subfigure}
\caption{\textbf{Rank sensitivity on HaluEval (GPT-2).}
Transcoders and Top-K SAEs exhibit different rank-dependent trade-offs, supporting explainer-specific validation rather than a universal rank choice.
Orange circles show nComp (left axis), and cyan triangles show $|\Delta\mathrm{CE}|$ (right axis), with axis ranges differing between panels.
Dotted lines mark the ranks used in the main experiments.}
\label{fig:rank_sensitivity}
\end{figure}

\textbf{Rank $r$}: For GPT-2 transcoders on HaluEval, nComp varies little across the evaluated ranks, while reconstruction quality is more sensitive.
At $r=32$, nComp is $0.963$ and $|\Delta\mathrm{CE}|$ is $0.0014$; extending the sweep to $r=128$ gives $0.954$ and $0.0058$, respectively.
The Top-K SAE sweep shows a different trade-off: $r=256$ gives the highest nComp ($1.412$), whereas $r=640$ gives the lowest $|\Delta\mathrm{CE}|$ ($0.0049$).
Thus, $r$ controls the retained OOD geometry rather than representing a universal intrinsic dimension of distribution shift.
We select rank on a validation split by jointly monitoring nComp and $|\Delta\mathrm{CE}|$, reusing the same OOD activations and eigenvectors across ranks.

\textbf{OOD sample size $N_{\mathrm{OOD}}$}: Increasing the number of OOD activation vectors does not monotonically improve the downstream metrics.
From $N_{\mathrm{OOD}}=2{,}048$ to $4{,}096$, nComp changes little ($0.963$ to $0.967$), but $|\Delta\mathrm{CE}|$ increases from $0.0014$ to $0.0121$.
At $N_{\mathrm{OOD}}=8{,}192$ and $16{,}384$, $|\Delta\mathrm{CE}|$ returns to $0.0018$ and $0.0017$, respectively, while nComp is $0.966$ at both sample sizes.
Although $N_{\mathrm{OOD}}=512$ gives the highest nComp ($1.034$), it also gives the largest reconstruction drift ($0.0383$).
These results support $N_{\mathrm{OOD}}=2{,}048$ as a favorable balance in this setting, rather than a universally optimal sample size.

\textbf{Dictionary preservation weight $\lambda_{\mathrm{pres}}$}: increasing $\lambda_{\mathrm{pres}}$ trades a small nComp decrease ($0.02$) for a large $|\Delta\mathrm{CE}|$ improvement ($0.026\!\to\!0.001$), with the default near the elbow.

\section{Faithfulness on Held-out In-Distribution Data}
\label{sec:id_appendix}

We evaluate GAE against the two training-free baselines (Fixed, TERM) on a held-out subset of each model's training corpus (OpenWebText for GPT-2, the Pile for Pythia-1.4B). Training-based baselines (Retrain, Finetune, SAEBoost, FaithfulSAE) are outside this comparison because they target a different operating regime, exploiting sample-level specialization rather than geometric adaptation, so mixing them would confound the cause of any improvement.

Each model is paired with a held-out slice of its own training corpus that no explainer has touched during dictionary fitting. The evaluation prompts and the 2{,}000-sample adaptation set are drawn fresh from this slice, sharing the same broad domain as the data the Fixed explainer was trained on while remaining strictly unseen. The resulting setup isolates the no-shift regime at the distribution level: there is no domain gap and no temporal gap, only the sampling variation that any finite subset inherits from a large corpus. The question this section asks is whether a geometric adaptation method has anything to do once the obvious shift has been removed.

\begin{figure}[h]
\centering
\begin{subfigure}{0.48\textwidth}
  \centering
  \includegraphics[width=\linewidth]{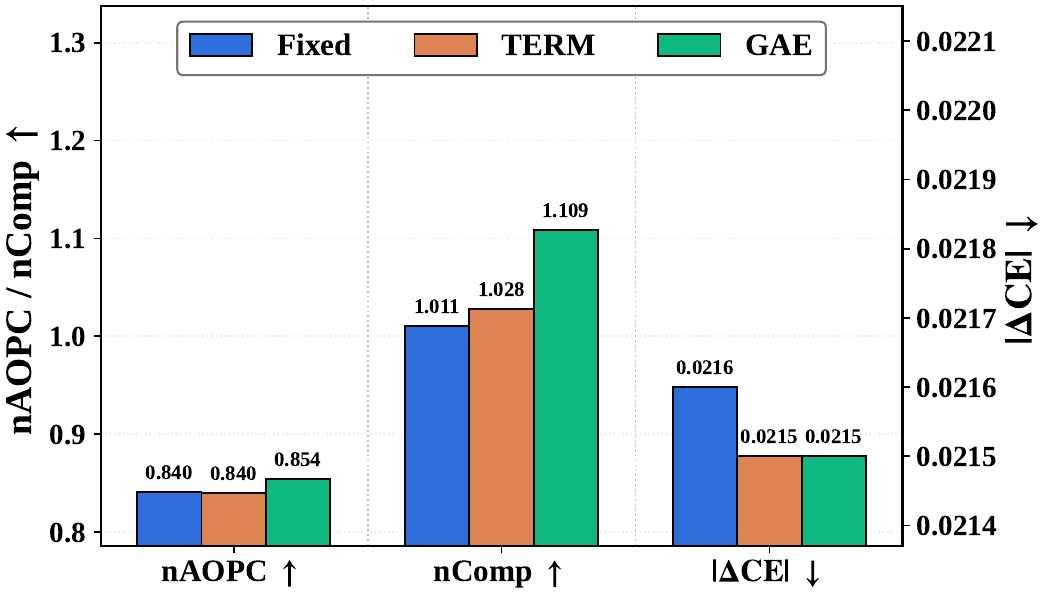}
  \caption{GPT-2}
\end{subfigure}\hfill
\begin{subfigure}{0.48\textwidth}
  \centering
  \includegraphics[width=\linewidth]{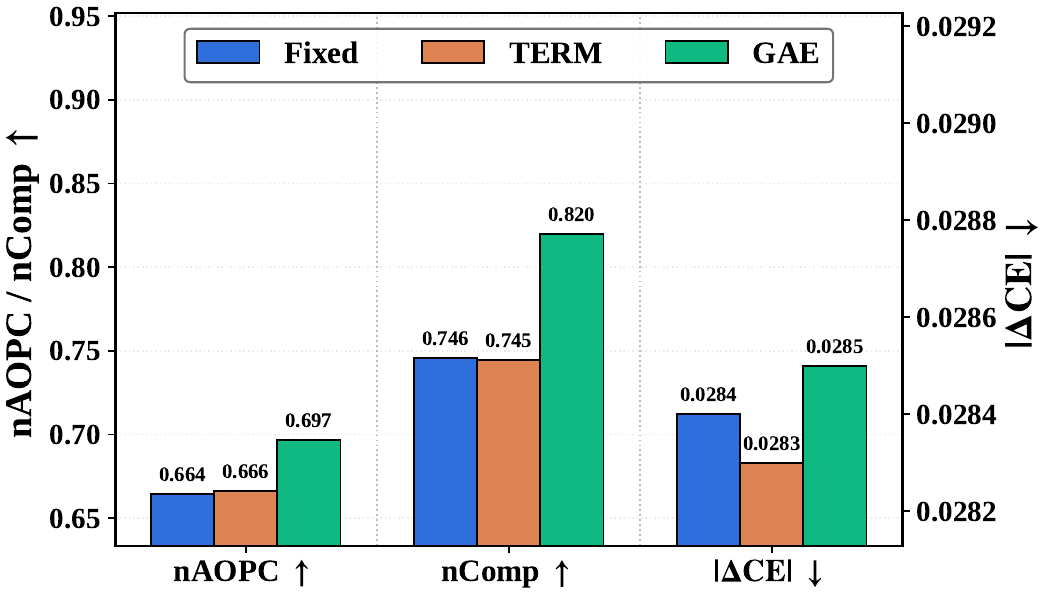}
  \caption{Pythia-1.4B}
\end{subfigure}
\caption{\textbf{Faithfulness of training-free explainer methods on held-out in-distribution data with the Transcoder explainer.} The left axis plots the causal-faithfulness metrics nAOPC and nComp (higher is better) and the right axis plots reconstruction quality $|\Delta\mathrm{CE}|$ (lower is better). Both backbones show the same pattern: GAE lifts nAOPC and nComp above Fixed and TERM, with the largest swing on GPT-2 (nComp $+0.10$), while $|\Delta\mathrm{CE}|$ tracks Fixed to within $0.0001$.}
\label{fig:id_appendix_transcoder}
\end{figure}

\begin{figure}[h]
\centering
\begin{subfigure}{0.48\textwidth}
  \centering
  \includegraphics[width=\linewidth]{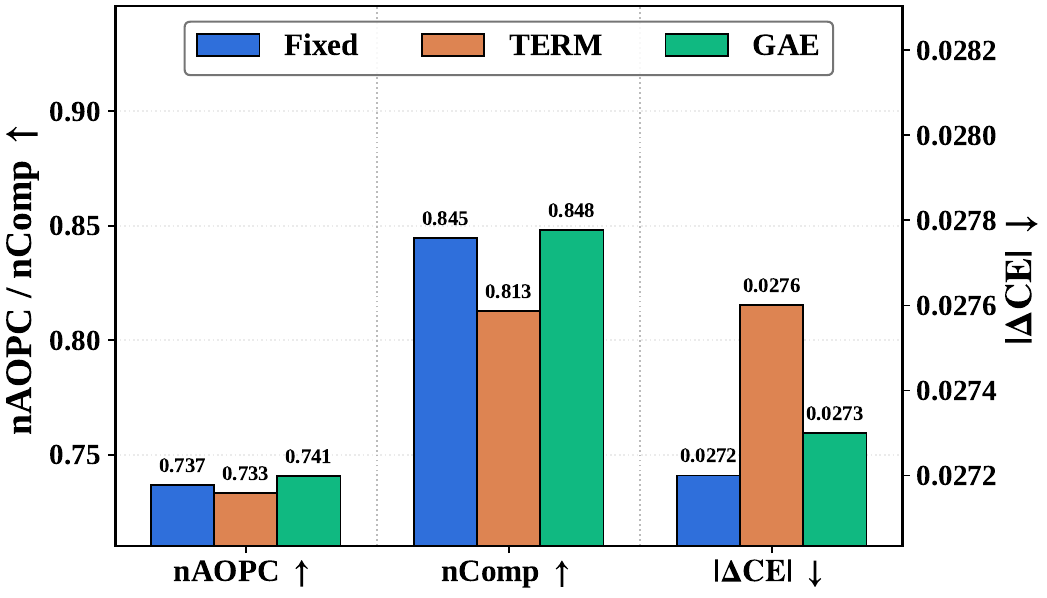}
  \caption{GPT-2}
\end{subfigure}\hfill
\begin{subfigure}{0.48\textwidth}
  \centering
  \includegraphics[width=\linewidth]{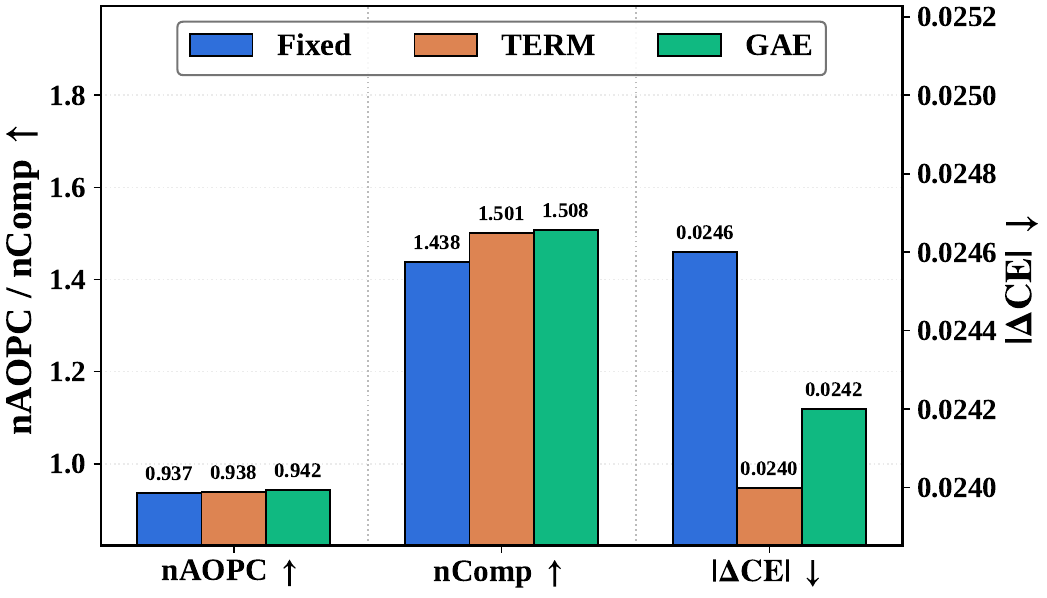}
  \caption{Pythia-1.4B}
\end{subfigure}
\caption{\textbf{Faithfulness of training-free explainer methods on held-out in-distribution data with the Top-K SAE explainer.} The left axis plots nAOPC and nComp (higher is better) and the right axis plots $|\Delta\mathrm{CE}|$ (lower is better). The SAE dictionary already sits closer to optimal on this ID slice, so the gap to Fixed is narrower than on the Transcoder cells, but GAE still moves both causal metrics in the right direction (largest at Pythia-1.4B with nComp $+0.07$), and on Pythia-1.4B even nudges $|\Delta\mathrm{CE}|$ slightly below Fixed ($0.0246 \to 0.0242$).}
\label{fig:id_appendix_sae}
\end{figure}

The answer is yes.
Even on this ID slice, the 2{,}000 adaptation samples carry their own subset-specific second-moment structure, which is statistically distinct from the full training corpus the Fixed checkpoint saw.
GAE picks up this finer geometry: it improves both causal metrics over Fixed in every cell, while leaving reconstruction quality nearly unchanged.
The Transcoder cells show the largest correction, with nComp moving from $1.011$ to $1.109$ on GPT-2 and from $0.746$ to $0.820$ on Pythia-1.4B, and nAOPC tracking ($0.840 \to 0.854$ and $0.664 \to 0.697$).
The Top-K SAE cells move less but in the same direction, with the largest jump at Pythia-1.4B where nComp rises from $1.438$ to $1.508$.
Reconstruction quality $|\Delta\mathrm{CE}|$ stays within $0.0005$ of Fixed on all four cells, so the causal-metric gains do not destabilize the well-aligned dictionary GAE starts from.
This is consistent with the geometric picture: GAE's correction is driven by the gap between the explainer's encoded covariance and the covariance of the evaluated activations, and that gap does not disappear simply because the two are drawn from the same nominal distribution.
The same mechanism that recovers faithfulness under explicit domain, temporal, and adversarial shift continues to operate, at smaller magnitude, on the residual within-corpus drift that ID held-out evaluation exposes.

\section{Evaluation on Llama-3.1-8B with LlamaScope}
\label{sec:larger_model}

\paragraph{Setup and faithfulness.}
We evaluate a public LlamaScope 32$\times$ Top-K SAE~\citep{he2024llamascope} on Llama-3.1-8B under the HaluEval protocol used in the main experiments.
The hidden dimension is $d=4{,}096$, the dictionary size is $k=131{,}072$, and the adaptation rank is $r=64$.
GAE uses $N=2{,}048$ unlabeled OOD hidden-state activation vectors while keeping the model and dictionary encoder frozen.
Table~\ref{tab:llama_results} compares GAE with Fixed, Finetune, and SAEBoost.
GAE achieves the highest nAOPC and nComp and the lowest $|\Delta\mathrm{CE}|$ among the evaluated methods.
Relative to Fixed, nAOPC increases from 0.5719 to 0.5972, nComp from 0.7892 to 0.7936, and $|\Delta\mathrm{CE}|$ decreases from 0.0652 to 0.0602.
This evaluation supports application to one larger public SAE on HaluEval.

\paragraph{Measured computational cost.}
Table~\ref{tab:llama_cost} compares GAE adaptation with OOD fine-tuning on an RTX A6000.
The GAE time includes second-moment estimation, eigenspace computation, and the closed-form decoder update for one rank, excluding hyperparameter search.
The fine-tuning run uses 199{,}936 OOD activation tokens, with 781 steps and batch size 256.
GAE reduces adaptation time substantially, while both methods use tens of gigabytes of GPU memory.

\begin{table}[ht]
\centering
\caption{\textbf{Computational cost for Llama-3.1-8B with LlamaScope on HaluEval.}
Time and GPU memory are approximate measurements on an RTX A6000.}
\label{tab:llama_cost}
{
\begin{tabular}{lcc}
\toprule
Method & Time & GPU memory \\
\midrule
OOD fine-tuning & ${\sim}$5\,h\,25\,min & ${\sim}$32\,GB \\
\rowcolor{gaerowhighlight} \textbf{GAE} & ${\sim}$\textbf{30\,s} & ${\sim}$\textbf{28\,GB} \\
\bottomrule
\end{tabular}
}
\end{table}

\paragraph{Scaling considerations.}
Forming the dense $d\times d$ OOD second-moment matrix and computing its eigenspace requires $O(Nd^2+d^3)$ time and $O(d^2)$ storage for the matrix.
The decoder itself contains $dk$ entries, so increasing the dictionary size also increases decoder storage and update cost.
GAE avoids gradient updates and optimizer states, but dense eigenspace computation remains a bottleneck as $d$ grows.
For rank sweeps, the OOD second moment and eigenvectors can be reused, with each additional rank changing the cutoff and decoder update.
Streaming second-moment accumulation and scalable eigensolvers are potential extensions for larger hidden dimensions.


\clearpage
\section*{NeurIPS Paper Checklist}

\begin{enumerate}

\item {\bf Claims}
    \item[] Question: Do the main claims made in the abstract and introduction accurately reflect the paper's contributions and scope?
    \item[] Answer: \answerYes{}{} 
    \item[] Justification: The main claims in the abstract and introduction accurately reflect the paper’s contributions: formalizing the faithfulness gap, proposing GAE, proving theoretical guarantees, and evaluating the method across multiple OOD settings.
    The scope is also qualified by the evaluated models and shifts, with Llama-3.1-8B evaluated only on HaluEval (Appendix~\ref{sec:larger_model}).

\item {\bf Limitations}
    \item[] Question: Does the paper discuss the limitations of the work performed by the authors?
    \item[] Answer: \answerYes{} 
    \item[] Justification: The paper discusses limitations in the conclusion, including the top-$r$ SVD truncation of the ID dictionary, and future directions such as adaptive rank selection and encoder adaptation.
    Appendix~\ref{sec:larger_model} discusses the cost of dense eigenspace computation as the hidden dimension grows.

\item {\bf Theory assumptions and proofs}
    \item[] Question: For each theoretical result, does the paper provide the full set of assumptions and a complete (and correct) proof?
    \item[] Answer: \answerYes{} 
    \item[] Justification: The paper states assumptions for its main theoretical results, including eigengap assumptions for Propositions 1 and 2 and Theorem 1. Complete proofs and derivations are provided in Appendix A.

    \item {\bf Experimental result reproducibility}
    \item[] Question: Does the paper fully disclose all the information needed to reproduce the main experimental results of the paper to the extent that it affects the main claims and/or conclusions of the paper (regardless of whether the code and data are provided or not)?
    \item[] Answer: \answerYes{} 
    \item[] Justification: The paper provides the model/explainer configurations, OOD settings, baselines, evaluation metrics, GAE hyperparameters, evaluation protocol, and evaluation size needed to reproduce the main experimental comparisons.
    Appendix D gives the detailed experimental and implementation setup, and the public code repository is available at \url{https://github.com/MLAI-Yonsei/GAE/} to further support reproduction.

\item {\bf Open access to data and code}
    \item[] Question: Does the paper provide open access to the data and code, with sufficient instructions to faithfully reproduce the main experimental results, as described in supplemental material?
    \item[] Answer: \answerYes{} 
    \item[] Justification: The paper provides a public code repository at \url{https://github.com/MLAI-Yonsei/GAE/}.
    The repository README documents the implementation framework, main runner, supported models and datasets, baselines, shared adaptation/evaluation protocol, implemented metrics, CLI entrypoints, output contracts, and reproducibility checks; setting-specific OOD details are documented in separate repository files.

\item {\bf Experimental setting/details}
    \item[] Question: Does the paper specify all the training and test details (e.g., data splits, hyperparameters, how they were chosen, type of optimizer) necessary to understand the results?
    \item[] Answer: \answerYes{} 
    \item[] Justification: The experimental setup specifies the target models, explainer types, dictionary sizes, ID and OOD datasets, baselines, evaluation metrics, feature budgets, evaluation protocol, and evaluation size. Appendix D provides additional model/explainer configurations, baseline descriptions, GAE implementation details, and GAE hyperparameters.

\item {\bf Experiment statistical significance}
    \item[] Question: Does the paper report error bars suitably and correctly defined or other appropriate information about the statistical significance of the experiments?
    \item[] Answer: \answerNo{} 
    \item[] Justification: We report metrics as the mean over 1{,}000 evaluation prompts per setting, without error bars or confidence intervals for these main comparisons.
    Table~\ref{tab:concept_retention} reports paired-bootstrap 95\% confidence intervals for the concept-level DLA differences.
    We use a single training seed (2026) for stochastic baselines due to compute constraints (Appendix~\ref{subsec:compute_resources}).

\item {\bf Experiments compute resources}
    \item[] Question: For each experiment, does the paper provide sufficient information on the computer resources (type of compute workers, memory, time of execution) needed to reproduce the experiments?
    \item[] Answer: \answerYes{} 
    \item[] Justification: We provide information about the computer resources in Appendix~\ref{subsec:compute_resources}.
    
\item {\bf Code of ethics}
    \item[] Question: Does the research conducted in the paper conform, in every respect, with the NeurIPS Code of Ethics \url{https://neurips.cc/public/EthicsGuidelines}?
    \item[] Answer: \answerYes{} 
    \item[] Justification: The authors have reviewed the NeurIPS Code of Ethics and are not aware of any deviations. The work evaluates mechanistic interpretability methods on existing pretrained language models and existing text datasets/benchmarks, without human-subject experiments, crowdsourcing, or deployment on users.

\item {\bf Broader impacts}
    \item[] Question: Does the paper discuss both potential positive societal impacts and negative societal impacts of the work performed?
    \item[] Answer: \answerNo{} 
    \item[] Justification: The current paper is primarily a methodological study of faithful dictionary-based interpretability under distribution shift and does not include a dedicated discussion of both positive and negative societal impacts. Because the method is evaluated as a research tool on existing frozen pretrained language models and existing text datasets/benchmarks, without a user-facing deployment, human-subject study, or release of a new high-risk model or dataset, the paper does not provide an application-specific societal-impact analysis; broader-impact considerations are left for future versions.
    
\item {\bf Safeguards}
    \item[] Question: Does the paper describe safeguards that have been put in place for responsible release of data or models that have a high risk for misuse (e.g., pre-trained language models, image generators, or scraped datasets)?
    \item[] Answer: \answerNA{} 
    \item[] Justification: The paper does not release a new pretrained language model, image generator, scraped dataset, or other high-risk artifact requiring controlled release safeguards. The released asset is an implementation of the proposed explainer adaptation method.

\item {\bf Licenses for existing assets}
    \item[] Question: Are the creators or original owners of assets (e.g., code, data, models), used in the paper, properly credited and are the license and terms of use explicitly mentioned and properly respected?
    \item[] Answer: \answerYes{} 
    \item[] Justification: The paper uses existing pretrained models, datasets, and benchmarks, and credits the original creators of the main experimental assets through citations. The authors use third-party assets in accordance with their applicable licenses and terms of use, without repackaging or redistributing restricted model checkpoints or datasets. The paper submission is designated under CC-BY 4.0.

\item {\bf New assets}
    \item[] Question: Are new assets introduced in the paper well documented and is the documentation provided alongside the assets?
    \item[] Answer: \answerYes{} 
    \item[] Justification: The main new asset is the released implementation of GAE, available at \url{https://github.com/MLAI-Yonsei/GAE/}.
    The method is documented in the paper through Algorithm 1 and Appendix D, including implementation details and reported hyperparameters; the accompanying repository documentation further provides protocol, execution, output, and reproducibility information.

\item {\bf Crowdsourcing and research with human subjects}
    \item[] Question: For crowdsourcing experiments and research with human subjects, does the paper include the full text of instructions given to participants and screenshots, if applicable, as well as details about compensation (if any)? 
    \item[] Answer: \answerNA{} 
    \item[] Justification: The paper does not involve crowdsourcing, user studies, or other human-subject experiments. All experiments are conducted on pretrained language models and existing text datasets.

\item {\bf Institutional review board (IRB) approvals or equivalent for research with human subjects}
    \item[] Question: Does the paper describe potential risks incurred by study participants, whether such risks were disclosed to the subjects, and whether Institutional Review Board (IRB) approvals (or an equivalent approval/review based on the requirements of your country or institution) were obtained?
    \item[] Answer: \answerNA{} 
    \item[] Justification: The paper does not involve human-subject research, crowdsourcing, or participant data collection, so IRB approval or equivalent review is not applicable.

\item {\bf Declaration of LLM usage}
    \item[] Question: Does the paper describe the usage of LLMs if it is an important, original, or non-standard component of the core methods in this research? Note that if the LLM is used only for writing, editing, or formatting purposes and does \emph{not} impact the core methodology, scientific rigor, or originality of the research, declaration is not required.
    \item[] Answer: \answerYes{} 
    \item[] Justification: The paper explicitly describes the use of pretrained language models, including GPT-2 Small and Pythia-1.4B, as frozen target models for evaluating dictionary-based explainers. LLMs are not used as a non-standard component for data annotation, data generation, method construction, or evaluation beyond this stated model-evaluation setting.

\end{enumerate}

\end{document}